\documentclass[11pt,letterpaper]{mystyle}

\usepackage[utf8]{inputenc}
\usepackage[T1]{fontenc}
\usepackage{xcolor}
\usepackage{float}
\usepackage{listings}

\definecolor{colorfirstBase}{RGB}{243, 152, 150}   % 原始红色
\definecolor{colorsecondBase}{RGB}{255, 203, 153}  % 原始橙色
\definecolor{colorthirdBase}{RGB}{254, 255, 152}   % 原始黄色

\colorlet{colorfirst}{colorfirstBase!40}   % 25% 透明度 (保留75%原色)
\colorlet{colorsecond}{colorsecondBase!30} % 25% 透明度
\colorlet{colorthird}{colorthirdBase!30}   % 25% 透明度

\definecolor{oursgreen}{RGB}{46,139,87}
\definecolor{vanillagray}{RGB}{108,117,125}
\definecolor{aflowred}{RGB}{196,78,82}
\definecolor{searchorange}{RGB}{221,132,82}
\definecolor{softblue}{RGB}{54,98,148}
\definecolor{softgold}{RGB}{181,142,52}
\definecolor{lightgray}{RGB}{246,247,249}
\definecolor{lightgreen}{RGB}{236,247,240}
\definecolor{lightred}{RGB}{252,239,240}
\definecolor{lightblue}{RGB}{236,242,250}

\newcommand{\method}{VOI}
\newcommand{\vanilla}{BAVT}
\newcommand{\bats}{BATS}
\newcommand{\aflow}{AFlow}
\newcommand{\searcho}{Search-o1}
\newcommand{\search}{\ifmmode\text{{\normalfont\scshape Search}}\else{\normalfont\scshape Search}\fi}
\newcommand{\decompose}{\ifmmode\text{{\normalfont\scshape Decompose}}\else{\normalfont\scshape Decompose}\fi}
\newcommand{\answer}{\ifmmode\text{{\normalfont\scshape Answer}}\else{\normalfont\scshape Answer}\fi}

\newcommand{\hotpot}{HotpotQA}
\newcommand{\twowiki}{2WikiMultihopQA}
\newcommand{\musique}{MuSiQue}
\newcommand{\bamboogle}{Bamboogle}
\newcommand{\budgetlow}{\texttt{low}}
\newcommand{\budgetlowermid}{\texttt{lower-mid}}
\newcommand{\budgetuppermid}{\texttt{upper-mid}}
\newcommand{\budgethigh}{\texttt{high}}
\newcommand{\cmark}{\ding{51}}
\newcommand{\xmark}{\ding{55}}
\newcommand{\circnum}[1]{\tikz[baseline=(char.base)]\node[shape=circle,draw=black!70,inner sep=0.5pt,minimum size=1.55em,font=\scriptsize\bfseries] (char) {#1};}
\newlist{contriblist}{enumerate}{1}
\setlist[contriblist]{label=\protect\circnum{\arabic*},leftmargin=*,itemsep=0.3em,topsep=0.35em,parsep=0pt,partopsep=0pt}
\newcolumntype{P}[1]{>{\raggedright\arraybackslash}p{#1}}
\newcolumntype{C}[1]{>{\centering\arraybackslash}p{#1}}

\lstdefinestyle{appendixpromptstyle}{
  basicstyle=\ttfamily\scriptsize,
  breaklines=true,
  breakatwhitespace=false,
  columns=fullflexible,
  keepspaces=true,
  frame=single,
  framerule=0.4pt,
  rulecolor=\color{black!35},
  framesep=2mm,
  xleftmargin=0mm,
  xrightmargin=0mm,
  showstringspaces=false,
  upquote=true,
}

\newcommand{\AppendixPromptTitle}[1]{%
  \par\vspace{0.8em}\noindent
  \colorbox{black!70!white}{%
    \parbox{\dimexpr\linewidth-2\fboxsep\relax}{\color{white}\textbf{\small #1}}%
  }%
  \par\nobreak\vspace{0.2em}%
}

\title{Inference-Time Budget Control for \\LLM Search Agents}
\runningtitle{Budget Control for LLM Search Agents}

\author{Zhengru Fang et al.}

\newcommand{\paperauthors}{%
Zhengru Fang$^{1,*}$,
Senkang Hu$^{1,*}$,
Zhonghao Chang$^{2}$,
Yu Guo$^{1}$,
Yihang Tao$^{1}$,
Hongyao Liu$^{1}$,
Mengzhe Ruan$^{3}$,
Jun Huang$^{1}$,
Yuguang Fang$^{1}$\\[0.3em]
\normalfont
$^1$ City University of Hong Kong,
$^2$ Tsinghua University,
$^3$ Alibaba Ant Group,
* Equal contribution%
}

\makeatletter
\renewcommand{\maketitle}{%
\vphantom{a}
\vspace{0mm}
\noindent
\begin{minipage}[t]{0.5\textwidth}
    \includegraphics[height=1.5cm]{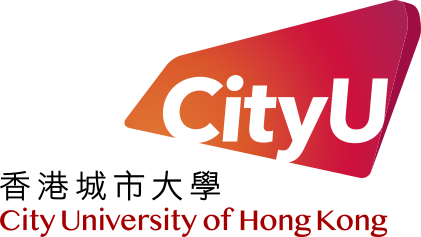}
\end{minipage}%
\hfill
\begin{minipage}[t]{0.5\textwidth}
    \raggedleft
\end{minipage}

\vspace{-5mm}
\noindent\rule{\textwidth}{0.4pt}
\vspace{-10pt}
\vspace{-10pt}

\begin{tcolorbox}[titlebox]
{\raggedright \titlefont \textbf{\@title}\par}
\vspace{4mm}
{\raggedright \paperauthors\par}
\vspace{4mm}
\abscontent
\end{tcolorbox}
\suppressfloats[t]
}
\makeatother

\begin{document}
\begin{abstract}
LLM search agents increasingly rely on tools at inference time, but their trajectories are often constrained by hard limits on both tool calls and generated tokens. Under such dual budgets, better answers require not only stronger models, but also explicit control over which search action should receive the next budget unit and when the accumulated evidence is sufficient to commit a final answer. We study this problem in multi-hop question answering (QA) and formulate it as two-stage inference-time budget control. At search time, our controller assigns each feasible action a task-level Value-of-Information (VOI) score, defined as an operational estimate of marginal task value per unit budget under the current search state and remaining dual budget, and uses this score to choose among retrieval, decomposition, and answer commitment. After search, a selective evidence-grounded finalizer compares the trajectory answer with a refined candidate and rewrites only when the residual error appears to be a low-risk answer-form error. Across four multi-hop QA benchmarks, three LLM backbones, and four budget levels, the method yields positive aggregate gains over four audited baselines under the same hard dual-budget protocol. Ablations show that search-time budget control, especially budget-dependent penalty, provides the main performance gain, while answer-time control helps mainly when the retrieval path is already adequate. These results suggest that inference-time budget control for LLM search agents should govern both how budget is spent during search and how the final answer is committed.
\end{abstract}

\maketitle
\pagestyle{headstyle}
\thispagestyle{empty}

\linenumbers
\section{Introduction}
\label{sec:intro}

% ============================================================
% BLOCK 1: Background — tool-augmented agents and test-time scaling
% ============================================================

Large language models (LLMs) are increasingly deployed as search agents that use external tools during inference. ReAct~\citep{yao2022react}, Toolformer~\citep{schick2023toolformer}, BATS~\citep{bats}, and Search-o1~\citep{DBLP:journals/corr/abs-2501-05366} treat retrieval as part of the inference trajectory rather than a fixed preprocessing step. Broader agent systems add reflection, tree search, collaboration, and software tools~\citep{shinn2023reflexion,yao2023tree,chen2023agentverse,wang2023voyager,zhou2023webarena,zhou2023language,lu2025octotools}. In this paradigm, answer quality depends not only on model capability, but also on how the agent spends tool-call and token budgets while inference is still unfolding.

This paradigm also changes test-time scaling from a token-only problem into an inference-time control problem. Prior work studies how additional inference compute improves outputs through repeated sampling, self-refinement, and adaptive allocation~\citep{wang2022self,DBLP:conf/nips/MadaanTGHGW0DPY23,snell2025scaling,han2025token,li2025selfbudgeter}. For search agents, however, compute is not token-only: extra budget can create more branches, intermediate states, and low-yield actions. Systems such as Search-R1~\citep{jin2025search}, BATS~\citep{bats}, and recent browsing agents~\citep{zhu2025scaling,wei2025browsecomp,thinking_doing,hu2026inforeasoner} already operate in this tool-heavy regime. Deployment studies likewise show that unconstrained tool use leads to budget-inefficient failures, including redundant retrieval, wasted branching, and brittle context growth~\citep{chen2023frugalgpt,zhang2023ecoassistant,kim2025cost,lu2025exploring}. Under explicit budgets, more inference-time compute and better answers are no longer equivalent.

% ============================================================
% BLOCK 2: Problem statement + Why the problem is hard
% ============================================================

This raises the central inference-time control problem of \emph{how to decide, under explicit tool-call and output-token budgets, what to do next, and how to commit a final answer once search stops.} Existing budget-aware approaches either route across models to reduce cost~\citep{chen2023frugalgpt,zhang2023ecoassistant} or inject budget tracking into the agent loop~\citep{bats,cats}, but none of them explicitly control both action allocation during search and answer resolution after search under the same hard dual-budget audit, which motivates this work. Three difficulties make this control problem nontrivial.

\textbf{Challenge 1: Budget must be allocated across heterogeneous actions.}
At each step a search agent can retrieve evidence, decompose the question, or answer directly. These actions differ in cost and return, and the right choice depends on the evidence state and remaining budget. Infrastructure analyses show that naive scaling leads to over-search and diminishing returns~\citep{kim2025cost,cemri2025multi}, while token-level budgeting cannot resolve this action-level trade-off under coupled tool-call and output-token budgets~\citep{han2025token,li2025selfbudgeter}.

\textbf{Challenge 2: Final-answer errors persist after successful search.}
Even with the right evidence, the final answer can fail on exactness: yes/no polarity, binary choice, typed slots, or alias completion. Self-refinement~\citep{DBLP:conf/nips/MadaanTGHGW0DPY23}, stepwise verification~\citep{DBLP:journals/corr/abs-2501-19306,DBLP:conf/cidr/0001YF0LH24}, and reasoning-aware selection~\citep{DBLP:conf/naacl/WanWCL25} can help, but they target open-ended generation rather than budgeted search where every extra call has a cost.

\textbf{Challenge 3: Rewriting introduces intervention risk.}
Unconditional answer rewriting can erase bridge structure or reverse comparative semantics in multi-hop settings~\citep{trivedi-etal-2023-interleaving,trivedi2022musique}. Reflexion-style feedback loops~\citep{shinn2023reflexion} mitigate some errors but do not explicitly model the risk of damaging a correct answer. Thus, an answer controller must intervene only when the expected gain outweighs risk.

% ============================================================
% BLOCK 3: Our approach and contributions (compact)
% ============================================================
To address these challenges, we propose a training-free, two-stage inference-time controller built on a tree-search backbone inspired by BAVT~\citep{bavt}. At search time, the controller scores each feasible action by task-level VOI: an operational estimate of marginal task value per unit budget under the current trajectory state and remaining dual budget, rather than Shannon information gain or Bayesian posterior value. The score combines critic-derived value change, structural signals, cost normalization, budget penalty, and conservative guards to choose among retrieval, decomposition, and answer commitment. After search, an evidence-grounded finalizer rewrites only when expected exactness gain outweighs the risk of damaging an otherwise adequate trajectory answer. The method therefore controls two distinct decision points: how budget is spent during search, and how the final answer is committed after search.
We evaluate \method{} under a shared hard dual-budget audit across four multi-hop QA benchmarks, three LLM backbones, and a four-level budget ladder. The results support explicit two-stage control while also exposing its boundary conditions across datasets, budgets, and backbones. Our main contributions are as follows:
\begin{itemize}[leftmargin=1.5em,itemsep=0.30em]
    \item[\ding{182}] \textbf{Two-stage inference-time budget-control formulation.} We formulate budget-aware inference in LLM search agents as a two-stage inference-time control problem: search-time action allocation under coupled tool-call and output-token budgets, followed by answer-time finalization under intervention risk.

    \item[\ding{183}] \textbf{Task-level VOI control for search actions.}
    We introduce a training-free scorer that ranks retrieval, decomposition, and answer commitment by estimated marginal task value per unit budget. The scorer combines critic-derived progress, structural signals, cost normalization, budget-dependent penalty, and conservative guards.

    \item[\ding{184}] \textbf{Risk-controlled answer finalization.}
    We introduce an evidence-grounded finalizer that rewrites only for low-risk answer-form errors, such as yes/no polarity, binary choices, typed-slot mismatches, or supported factoid completion, and abstains when bridge or comparative reasoning remains unresolved.

    \item[\ding{185}] \textbf{Empirical analysis under hard dual-budget audit.}
We evaluate across four benchmarks, three backbones, and four budget levels under the same hard dual-budget protocol. The results show positive aggregate gains and many low- and mid-budget improvements, while ablations identify budget-dependent penalty as the dominant component and expose backbone- and dataset-dependent boundary conditions where the gains diminish.

\end{itemize}

\section{Related Work}
\label{sec:related}

\textbf{Search agents and tool-augmented inference.}
Tool use shifted LLMs from static text generation to interactive decision-making. ReAct~\citep{yao2022react} interleaves reasoning and acting, Toolformer~\citep{schick2023toolformer} studies self-supervised tool use, and later systems add reflection, tree search, collaboration, environment grounding, and software tools~\citep{shinn2023reflexion,yao2023tree,chen2023agentverse,liuagentbench,wang2023voyager,zhou2023webarena,zhou2023language,qiao2024autoact,xie2024osworld,yang2024swe,lu2025octotools}. Agent evaluation has also expanded toward real-world and proactive-assistance settings~\citep{tang2026proactive}. Search-oriented agents such as Search-R1~\citep{jin2025search}, BATS~\citep{bats}, Search-o1~\citep{DBLP:journals/corr/abs-2501-05366}, BrowseComp~\citep{wei2025browsecomp}, and recent browsing systems~\citep{DBLP:journals/corr/abs-2505-22648,thinking_doing,DBLP:journals/corr/abs-2506-15741,DBLP:journals/corr/abs-2507-02592,DBLP:journals/corr/abs-2507-16075,DBLP:journals/corr/abs-2508-02694,DBLP:journals/corr/abs-2508-07976,DBLP:journals/corr/abs-2508-07999,DBLP:journals/corr/abs-2508-09129} push toward open-ended information seeking. We instead study QA agents under explicit tool-call and output-token budgets.

\begin{wrapfigure}[25]{r}{0.34\textwidth} \vspace{-0.8em} \centering \includegraphics[width=0.95\linewidth]{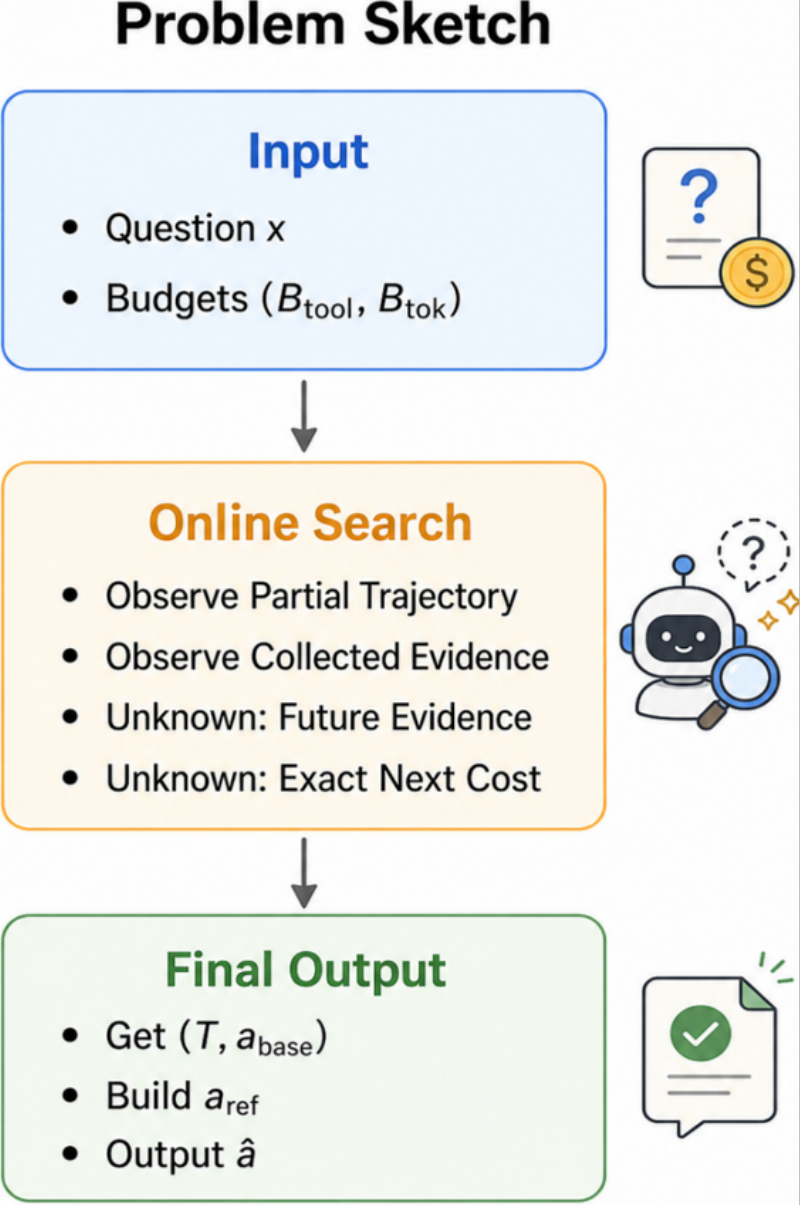} \caption{\textbf{Problem sketch.} The agent chooses the next search action from the observed trajectory, evidence, and remaining budget; future evidence and exact next cost are unknown.} \label{fig:problem_sketch} \vspace{-0.5em} \end{wrapfigure}

\textbf{Test-time scaling and budget-aware agent scaling.}
Test-time scaling work turns extra inference compute into better outputs through repeated sampling, self-refinement, early stopping, meta-generation, adaptive compute, efficient task adaptation, and token-budgeted reasoning~\citep{wang2022self,DBLP:conf/nips/MadaanTGHGW0DPY23,DBLP:conf/iclr/LiYFPW0W024,DBLP:journals/tmlr/WelleckBFSXNKH24,snell2025scaling,DBLP:conf/naacl/WanWCL25,han2025token,li2025selfbudgeter,hu2025distribution,DBLP:journals/corr/abs-2407-19825,DBLP:journals/corr/abs-2407-21787,DBLP:journals/corr/abs-2501-19306,DBLP:journals/corr/abs-2501-19393,DBLP:journals/corr/abs-2504-00762,DBLP:journals/corr/abs-2504-13367,DBLP:journals/corr/abs-2503-24235}. Budget-aware inference adds a harder question: not just how much extra compute to spend, but where. FrugalGPT~\citep{chen2023frugalgpt} and EcoAssistant~\citep{zhang2023ecoassistant} route across models, while agent-scaling work~\citep{zhu2025scaling,bats,kim2025cost,lu2025exploring} argues that tool use changes the scaling problem itself. Related edge and mobile intelligence studies further highlight that LLM and AI deployment is constrained by communication, caching, model downloading, continual adaptation, and split or federated execution costs~\citep{qu2025mobile,qu2026trimcaching,wu2026review,wu2026lifecycle,ding2026application}. We follow this resource-aware framing at a finer granularity: step-level action control during search and selective intervention at finalization.

\textbf{Workflow search and tree-based control.}
DSPy~\citep{omar2024dspy}, PromptAgent~\citep{xin2024promptagent}, Promptbreeder~\citep{chirs2024probreed}, ADAS~\citep{adas}, AgentSquare~\citep{agentsquare}, AFlow~\citep{aflow}, AutoFlow~\citep{ze2024autoflow}, EvoFlow~\citep{evoflow}, MermaidFlow~\citep{mermaidflow}, and HyEvo~\citep{hyevo} optimize prompts, module graphs, or workflows before deployment. Related LLM-based automatic algorithm-design work also uses language models to evolve heuristics or heuristic sets before inference~\citep{liu2024evolution,liu2026eoh}. A separate line treats inference as structured search over states, thoughts, or plans, as in Tree of Thoughts~\citep{yao2023tree}, Graph of Thoughts~\citep{besta2024graph}, Language Agent Tree Search~\citep{zhou2023language}, BAVT~\citep{bavt}, and CATS~\citep{cats}. Our method belongs to this second family: given an existing search backbone, it controls which action spends the next budget unit and how the final answer is resolved.

\textbf{Answer verification and selective intervention.}
Self-Refine~\citep{DBLP:conf/nips/MadaanTGHGW0DPY23} shows that post-hoc editing can improve outputs, while stepwise debugging~\citep{li2024debug}, VerifAI~\citep{DBLP:conf/cidr/0001YF0LH24}, Reasoning-Aware Self-Consistency~\citep{DBLP:conf/naacl/WanWCL25}, and SETS~\citep{DBLP:journals/corr/abs-2501-19306} show that revision quality depends on trigger timing and intervention risk. Our final-answer layer applies that lesson to budgeted search QA: the finalizer keeps the trajectory answer unless the expected gain clearly outweighs the intervention risk, since aggressive answer replacement can erase bridge structure or comparative semantics even when the trajectory is adequate.

\section{Problem Formulation}
\label{sec:problem}

We study inference-time budget control for an LLM search agent answering a question $x$ under dual budgets $B=(B_{\mathrm{tool}},B_{\mathrm{tok}})$, where $B_{\mathrm{tool}}$ limits tool calls and $B_{\mathrm{tok}}$ limits budgeted output tokens. The agent faces two coupled decisions. As illustrated in Figure~\ref{fig:problem_sketch}, the agent must choose the next search action from the currently observed trajectory, collected evidence, candidate answer state, and remaining budget, before future evidence and the exact next token cost are known. The agent faces two coupled decisions. During search, it must decide how to allocate the remaining budget across retrieval, decomposition, and answer commitment. After search terminates, it must decide whether to keep the trajectory answer or replace it with a refined answer derived from the collected evidence. The problem is therefore to control both budget allocation during search and answer commitment after search.

Let $\mu$ denote the search-time decision rule. Given $x$ and $B$, the search process produces a trajectory and base answer $(\mathcal{T},a_{\mathrm{base}})=\mathcal{M}(x,B;\mu).$
The decision rule is applied online. Before choosing the next operation, the agent observes the question, the partial trajectory, the evidence and decompositions collected so far, any candidate answer states, and the remaining tool-call and token budgets. It does not know future evidence, final answer quality, or the exact output-token cost of the next operation. Therefore, action selection uses an estimated local budget charge available before execution, while the realized tool-call and output-token costs are debited after the operation is executed. 

The completed trajectory $\mathcal{T}$ records the executed operations, collected evidence, intermediate states, candidate answers, and realized costs.
After the trajectory is complete, a refinement function $g_{\mathrm{ref}}$ constructs a candidate refined answer $a_{\mathrm{ref}}=g_{\mathrm{ref}}(\mathcal{T},a_{\mathrm{base}}),$
and a second policy $\sigma$ chooses the final prediction $\hat a$ from $\{a_{\mathrm{base}},a_{\mathrm{ref}}\}$. This stage is intentionally narrow: it is not unrestricted rewriting, but selective intervention between preserving the trajectory answer and replacing it with a refined one. This distinction matters in multi-hop QA, where an unnecessary rewrite can remove bridge entities, alter comparison structure, or replace a precise answer with a broader but less faithful one.
We use $R(a,y)\in[0,1]$ as an abstract bounded QA reward, where $a$ is a predicted answer and $y$ is the gold answer. In experiments, answer quality is reported using exact match and token-level F1; the formulation only requires a bounded scalar reward. The expectation is over the evaluation distribution of $(x,y)$ and the search trajectory induced by the policies. The cost functions $C_{\mathrm{tool}}(\mathcal{T})$ and $C_{\mathrm{tok}}(\mathcal{T})$ denote realized trajectory costs.
We formulate the overall problem as
\begin{equation}
\max_{\mu,\sigma}\; \mathbb{E}\!\left[R(\hat a,y)\right]
\quad \text{s.t.} \quad
C_{\mathrm{tool}}(\mathcal{T}) \le B_{\mathrm{tool}},\;
C_{\mathrm{tok}}(\mathcal{T}) \le B_{\mathrm{tok}},\;
\mathbb{E}\!\left[L_{\mathrm{harm}}(\hat a,a_{\mathrm{base}},y)\right] \le \rho_{\mathrm{harm}},
\label{eq:two_stage_problem}
\end{equation}
where $\rho_{\mathrm{harm}}$ is the allowable expected harm from answer replacement.

The harm term measures the damage caused by replacing the base answer with a worse final answer:
\[
L_{\mathrm{harm}}(\hat a,a_{\mathrm{base}},y)
=
\max\{0,R(a_{\mathrm{base}},y)-R(\hat a,y)\}.
\]
This quantity is zero when the final decision preserves or improves answer quality and positive only when finalization makes the answer worse than the original trajectory answer. Under this formulation, the first stage is an online budget-allocation problem during search, and the second stage is a risk-controlled answer-finalization problem after search. Our method instantiates these two decisions using a search-time controller and an answer-time selector with abstention.

\section{Method}
\label{sec:method}
\begin{figure*}[t]
\centering
\includegraphics[width=\textwidth]{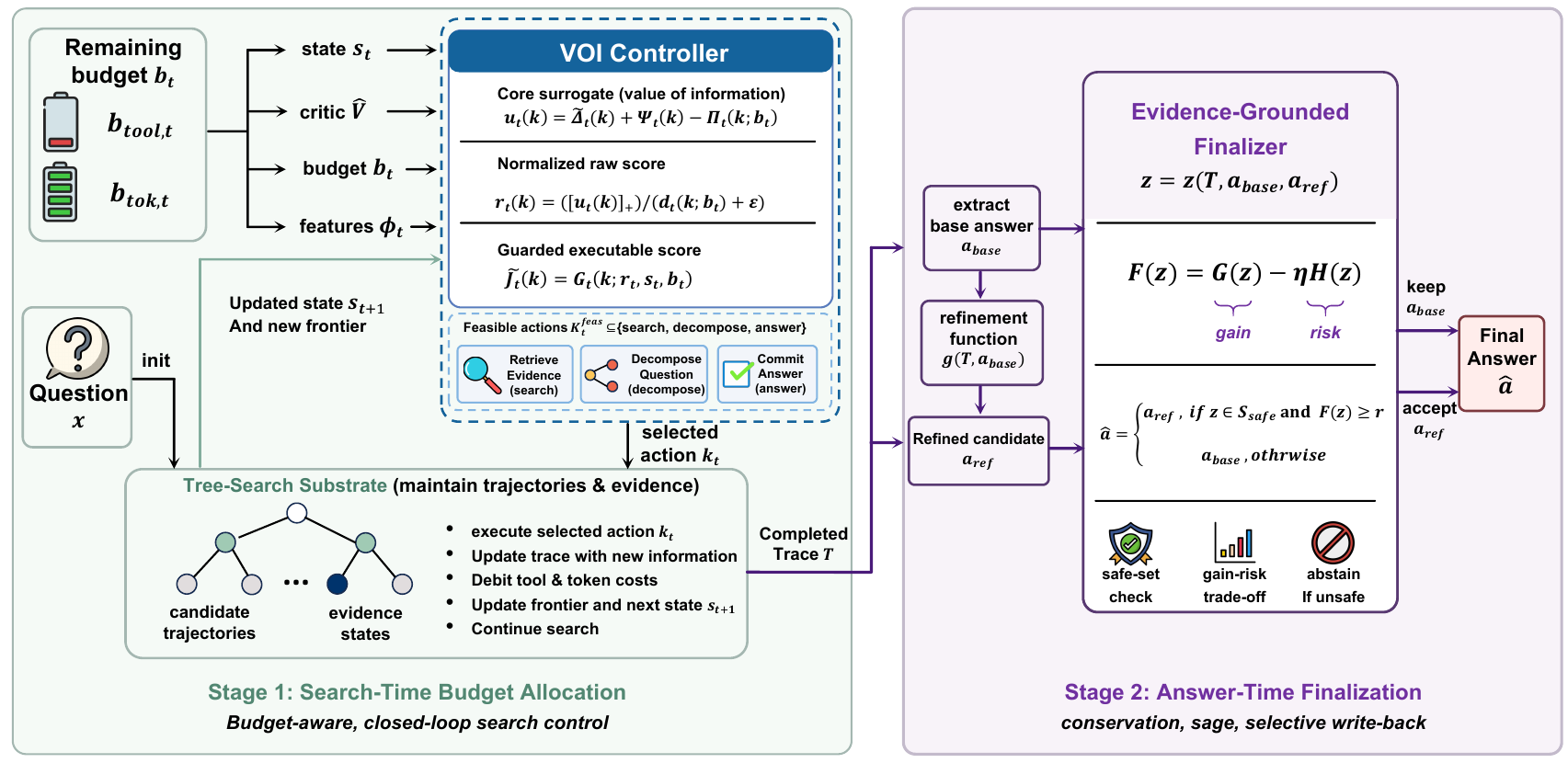}
\caption{\textbf{Two-stage budget control with task-level \method{}.}
Stage 1 uses a controller based on the task-level VOI score to choose whether the next step should retrieve, decompose, or answer under the remaining dual budget. Stage 2 finalizes the answer conservatively, rewriting only when the case is safe and the expected gain outweighs rewrite risk.}
\label{fig:teaser}
\vspace{-6mm}
\end{figure*}

% \subsection{Overview}
% \label{subsec:overview}

Our method instantiates the two decisions in Section~\ref{sec:problem} with two lightweight layers on top of a generic search backbone. Figure~\ref{fig:teaser} summarizes the full pipeline. In Stage~1, the backbone maintains candidate trajectories and evidence, while the VOI controller scores feasible actions under the remaining dual budget and selects the next operation. The selected operation is executed by the tree-search substrate, which updates the trace and debits realized tool-call and token costs. In Stage~2, the completed trace yields a base answer and a refined candidate; the finalizer keeps the base answer unless a safe, evidence-supported refinement passes the gain--risk check.
The search-time layer chooses among evidence acquisition, question decomposition, and answer commitment under the remaining budget. The answer-time layer then chooses between a base answer $a_{\mathrm{base}}$ extracted from the trace and a refined answer $a_{\mathrm{ref}}$ constructed from the same trace. The final output is $\hat a$.

\subsection{Search-Time Budget Allocation}
\label{subsec:controller}

\paragraph{Task-level VOI control.}
At each search step, the controller ranks feasible actions using a task-level VOI score: an operational estimate of marginal task value per unit budget under the current search state and remaining dual budget. This score is used only for action selection; it is not Shannon information gain, Bayesian posterior value, or a learned value model.

We instantiate the online decision rule in Section~\ref{sec:problem} as follows. At step $t$, the observed search status is represented by $s_t$, the remaining budget by $b_t=(b_{\mathrm{tool},t},b_{\mathrm{tok},t})$, and the feasible action set by $\mathcal{K}^{\mathrm{feas}}_t \subseteq \mathcal{K}$, where $\mathcal{K}=\{\search,\decompose,\answer\}$. For each feasible action $k$, the controller assigns a pre-execution charge vector $g_t(k)=(g_{\mathrm{tool},t}(k),g_{\mathrm{tok},t}(k))^\top$ and extracts fixed features from $s_t$ summarizing unresolved evidence, compositional structure, answer readiness, stagnation, loop pressure, and premature-answer risk.

The controller scores each feasible action by first forming a budget-shaped utility, then normalizing it as value per budget, and finally applying guards. First, it forms
\begin{equation}
u_t(k)
=
\underbrace{\widehat{\Delta}_t(k)}_{\text{progress signal}}
+
\underbrace{\Psi_t(k)}_{\text{structural signals}}
-
\underbrace{\Pi_t(k;\,b_t)}_{\text{budget-dependent penalty}},
\label{eq:controller_score}
\end{equation}
where $\widehat{\Delta}_t(k)$ is a critic-derived progress signal, $\Psi_t(k)$ aggregates structural signals such as bridge structure, loop pressure, readiness, and early-answer risk, and $\Pi_t(k;b_t)$ is a signed budget-dependent penalty term: it is positive for \search{} and \decompose{} under tight budgets. $\Pi_t(k;b_t)$ serves as a tractable proxy for the oracle budget shadow-cost $\lambda_t^{\star\top}g_t(k)$, with approximation error $\varepsilon_t^{\Pi}$ formalized in Assumption~\ref{assump:score_approx}.
Second, this task-level action utility is clipped and normalized into the task-level VOI score used for action selection:
\begin{equation}
r_t(k)
=
\frac{[u_t(k)]_+}{d_t(k;\,b_t)+\epsilon},
\label{eq:normalized_score}
\end{equation}
where $[\cdot]_+=\max\{0,\cdot\}$, $\epsilon>0$ is a small constant, and $d_t(k;b_t)>0$ is an action-specific budget scale derived from the pre-execution charge and current budget state. We call $r_t(k)$ the task-level VOI score: an operational estimate of marginal task value per unit budget under the current state and remaining dual budget.
\begin{definition}[Task-level VOI score]
For feasible $k\in\mathcal{K}^{\mathrm{feas}}_t$, the task-level VOI score is $r_t(k)=[u_t(k)]_+/(d_t(k;b_t)+\epsilon)$, where $u_t(k)$ is the task-level action utility and $d_t(k;b_t)$ is the budget-aware action scale. This score estimates marginal task value per unit budget for action selection under the current search state and remaining dual budget.
\end{definition}
In the released implementation, $u_t(k)$ is instantiated by fixed coefficients over explicit trajectory features, $d_t(k;b_t)$ by action-specific budget-aware cost terms, and the executable score $\widetilde{\mathcal{J}}_t(k)$ by these normalized scores after conservative guard adjustments. 

Third, action-specific guards produce the executable score:
\begin{equation}
\widetilde{\mathcal{J}}_t(k)
=
\mathfrak{G}_t\!\bigl(k;\,r_t,\,s_t,\,b_t\bigr),
\label{eq:guarded_score}
\end{equation}
where $\mathfrak{G}_t$ is a deterministic guard operator that masks or adjusts raw scores, including premature-answer suppression, factoid-decomposition suppression, and minimum-search enforcement on compositional cases. The controller selects $k_t=\arg\max_{k\in\mathcal{K}_t^{\mathrm{feas}}}\widetilde{\mathcal{J}}_t(k)$.

This design is adaptive, training-free, and interpretable: scores are recomputed at every step, no extra value model is learned, and each term corresponds to an explicit search or budget signal. Appendix~\ref{app:mechanism} provides empirical controller diagnostics.

\subsection{Answer-Time Finalization}
\label{subsec:selector}

After search terminates, the system extracts a base answer $a_{\mathrm{base}}$ from the best trajectory and constructs a refined candidate $a_{\mathrm{ref}}$ from the same trace $\mathcal{T}$. Finalization is narrow but important: the search path may be largely correct while the final answer still contains local answer-form errors, such as wrong yes/no polarity, wrong binary choice, incomplete alias, or typed-slot mismatch. Unnecessary rewriting remains risky because fluent refinements can erase bridge entities, distort comparisons, or replace precise answers with broader but less faithful ones.

We therefore formulate finalization as selective intervention over the two-candidate set $\{a_{\mathrm{base}},a_{\mathrm{ref}}\}$. From the completed trace and the candidate pair, we construct a consistency feature vector
$z=z(\mathcal{T},a_{\mathrm{base}},a_{\mathrm{ref}})$.
This vector records evidence relevant to safe revision: support type, slot type, contradiction indicators, and unresolved bridge or comparative reasoning. For exposition, finalization is written as a gain--risk threshold over this feature representation; in implementation, it is a deterministic selector over explicit feature conditions, not a learned scorer or extra LLM call:
\begin{equation}
F(z)=G(z)-\eta H(z), \qquad
\hat a =
\begin{cases}
a_{\mathrm{ref}}, & \text{if } z\in\mathcal{S}_{\mathrm{safe}} \text{ and } F(z)\ge\tau,\\[2mm]
a_{\mathrm{base}}, & \text{otherwise},
\end{cases}
\label{eq:finalizer_score}
\end{equation}
where $G(z)$ measures the potential gain, $H(z)$ measures the intervention risk, $\eta>0$ controls risk aversion, $\tau$ is an abstention threshold, and $\mathcal{S}_{\mathrm{safe}}$ is the set of reliable revision cases. In the released implementation, the abstract map $g$ is instantiated by a fixed construction $a_{\mathrm{ref}}=g_{\mathrm{ref}}(\mathcal{T},a_{\mathrm{base}})$, and the selector is a deterministic rule over explicit feature conditions. Appendix~\ref{app:proofs} characterizes the exact rule in Proposition~\ref{prop:det_finalizer}.

The safe set $\mathcal{S}_{\mathrm{safe}}$ excludes unresolved bridge structure, unresolved comparative semantics, and missing direct support, where rewriting is most likely to harm a correct trajectory answer. Thus the finalizer is a conservative answer selector, not a generic editor: it intervenes only when the remaining error appears local to answer-form correction rather than retrieval or path selection. It adds no tool calls and issues no additional LLM call during finalization.

\subsection{End-to-End Inference}
\label{subsec:inference}

At inference time, the search procedure maintains a frontier of trajectories while the controller repeatedly selects the next feasible operation using Eqs.~\eqref{eq:controller_score}--\eqref{eq:guarded_score}. It debits realized tool-call and output-token costs after each operation and stops when the frontier or budget is exhausted, or when the search procedure's termination condition fires. The system then extracts $a_{\mathrm{base}}$, constructs $a_{\mathrm{ref}}$, and applies Eq.~\eqref{eq:finalizer_score}. Full pseudocode is in Appendix~\ref{app:algorithms}.

\paragraph{Theoretical support.}
Appendix~\ref{app:proofs} provides the formal support for the two controller layers. For search-time allocation, Appendix~\ref{app:search_time_proofs} shows that the task-level utility in Eq.~\eqref{eq:controller_score} locally approximates an oracle budget-charged one-step lookahead value, yielding ranking consistency under a margin condition and a one-step value-gap bound when guards are compatible with the utility ranking. For answer-time finalization, Appendix~\ref{app:answer_time_proofs} derives the gain--risk threshold in Eq.~\eqref{eq:finalizer_score} from the harm-constrained objective in Eq.~\eqref{eq:two_stage_problem}. The analysis supports budget-aware action selection and conservative answer replacement, but does not claim global optimality of the full search tree.

\section{Experiments}
\label{sec:exp}

\subsection{Setup}
\label{subsec:setup}

We instantiate the LLM search agent with three backbones: Qwen3-32B~\citep{yang2025qwen3}, Qwen3.5-122B~\citep{qwen3.5}, and GPT-5.4-Mini~\citep{openai2026gpt54mini}, chosen to span different model scales, families, and deployment regimes rather than tuning to a single LLM. We evaluate on four multi-hop QA benchmarks: \hotpot~\citep{yang2018hotpotqa}, \twowiki~\citep{ho2020constructing}, \musique~\citep{trivedi2022musique}, and \bamboogle~\citep{press2023measuring}. All methods use the same retrieval configuration: Search-R1 retrieval, question-only queries, and \texttt{top\_k=5}. We evaluate under four dual-budget levels, \budgetlow, \budgetlowermid, \budgetuppermid, and \budgethigh{}, corresponding to $(1,100)$, $(2,200)$, $(2,300)$, and $(3,500)$, where each pair denotes the tool-call cap and output-token cap.

The audit is strict at the example level: all five methods are scored under the same hard tool/output-token budget constraint, and any example exceeding either constraint is counted as failed. We compare \method{} with \vanilla{}~\citep{bavt}, \bats{}~\citep{bats}, \aflow{}~\citep{aflow}, and \searcho{}~\citep{DBLP:journals/corr/abs-2501-05366}. We report EM, token-level F1, average tool calls, and average budget output tokens; accounting details, cross-backbone confidence intervals, and feasible-only usage diagnostics are provided in Appendix~\ref{app:scope}.

\subsection{Main Results}
\label{subsec:main_results}

\begin{figure*}[t]
\centering
\includegraphics[width=\textwidth]{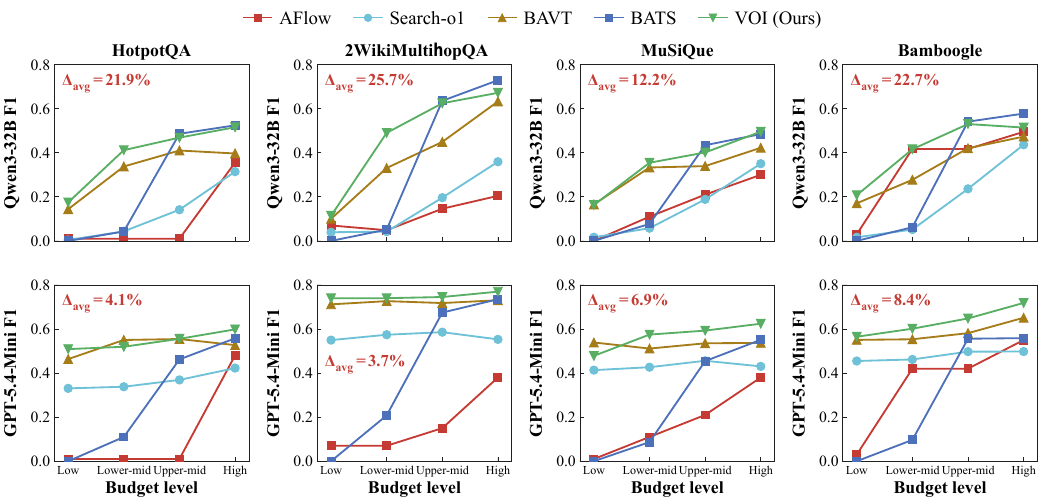}
\caption{\textbf{Cross-model budget scaling curves across four datasets.} Rows report Qwen3-32B and GPT-5.4-Mini; columns report the four multi-hop QA benchmarks. All methods use the shared hard dual-budget audit. Qwen3.5-122B results are reported in Appendix~\ref{app:qwen35_scaling} as a backbone-sensitivity analysis.}
\label{fig:scaling_cross_model}
\vspace{-7mm}
\end{figure*}
Figure~\ref{fig:scaling_cross_model} shows the main budget-scaling pattern on two representative backbones. On Qwen3-32B, whose full EM/F1 table is reported in Appendix~\ref{app:main_results_qwen32b}, \method{} is best-F1 in 7 of 16 dataset-budget cells and tied best in 2 additional cells. It improves over \aflow{}, \bats{}, \searcho{}, and \vanilla{} in 14/16, 10/16, 16/16, and 15/16 cells, respectively. The strongest gains appear in low and lower-mid budgets, where choosing the next action carefully is most important because wasted retrieval or premature answering quickly exhausts the trajectory. The pattern is positive but not a dominance claim. \bats{} remains strong in several upper-budget regimes, especially when broader beam-style evidence accumulation is useful, and \aflow{} leads on \twowiki{} \budgetlowermid{}. These exceptions are important because they show that budget control interacts with dataset structure and budget level. The main claim is therefore not that \method{} wins every cell, but that explicit action-level budget control improves most audited settings under the same hard budget constraints.

On GPT-5.4-Mini, \method{} remains consistently competitive despite the stronger base model. This is useful evidence that the controller is not merely compensating for a weak backbone. The gains are smaller in some cells than on Qwen3-32B, which is expected: stronger backbones can already resolve more answer-form and evidence-selection errors without as much controller help. Still, the curves show that the two-stage controller continues to provide value when the model family changes.

We report Qwen3.5-122B separately in Appendix~\ref{app:qwen35_scaling}. That backbone exhibits a more mixed regime, with several cells led by \bats{} or \vanilla{} at higher budgets. We treat it as a backbone-sensitivity case rather than the main trend. For completeness, Appendix~\ref{app:scope} reports cross-backbone macro deltas over all three backbones, four benchmarks, and four budgets; the aggregate effect remains favorable, while the detailed curves confirm substantial backbone dependence.

\subsection{Final-Answer Control Improves Exactness}
\label{subsec:selector_results}

Final-answer control is most useful when residual errors are answer-form errors rather than search-depth failures, clearest on \bamboogle. Evidence-grounded finalization improves the measured \bamboogle\ budget ladder without changing tool usage: F1 rises from 0.4382 to 0.4628 at \budgetlow, from 0.5786 to 0.6047 at \budgetuppermid, and from 0.5056 to 0.5576 at \budgethigh. An additional \texttt{(1,200)} probe shows the same trend, improving from 0.4631 to 0.4911 on \bamboogle\ and from 0.4907 to 0.5174 on \twowiki. Thus, answer-time control yields visible gains when the trajectory already contains adequate evidence but the final answer still has an answer-form error.

Successful Stage~2 interventions are typed-slot corrections, binary-choice repairs, yes/no polarity repairs, and supported factoid completions. This matches the design target: the finalizer is a conservative exactness layer, not a generic rewrite module. It creates no new evidence and launches no additional search or additional LLM call during finalization. Its gains therefore concentrate on answer-form errors, while path-level failures remain outside its scope. This also explains the weaker contribution on \musique{}, where residual errors more often involve unresolved bridge structure rather than a local answer span.

\subsection{Search-Time Controller Component Ablation}
\label{subsec:stage1_component_ablation}

Section~\ref{subsec:controller} defines the Stage~1 pipeline: task-level action utility, value-per-cost normalization, and guard-adjusted executable scoring in Eqs.~\eqref{eq:controller_score}--\eqref{eq:guarded_score}. We ablate these components under the \budgetuppermid{} budget on Qwen3-32B by removing budget-dependent penalty $\Pi_t(k;b_t)$, replacing $r_t(k)$ with $[u_t(k)]_+$, removing $\Psi_t(k)$, or bypassing $\mathfrak{G}_t$. All variants use the same search procedure, Search-R1 retrieval, question-only queries, \texttt{top\_k=5}, and hard dual-budget audit.

% \begin{table*}[t]
% \centering
% \scriptsize
% \setlength{\tabcolsep}{4.2pt}
% \renewcommand{\arraystretch}{1.08}
% \caption{\textbf{Stage-1 component ablation on Qwen3-32B (\budgetuppermid{} budget).}
% Each cell reports \texttt{EM/F1} under the hard dual-budget audit.}
% \label{tab:stage1_component_ablation_qwen32b_full}
% \resizebox{\textwidth}{!}{%
% \begin{tabular}{lcccccc}
% \toprule
% Benchmark & Vanilla BAVT & w/o penalty & w/o norm. & w/o struct. & w/o guards & VOI (Full) \\
% \midrule
% HotpotQA  & 0.34/0.41 & 0.35/0.40 & 0.34/0.38 & 0.32/0.37 & 0.34/0.39 & \textbf{0.39/0.47} \\
% \twowiki{}   & 0.40/0.45 & 0.38/0.43 & 0.44/0.50 & 0.39/0.45 & 0.42/0.48 & \textbf{0.54/0.63} \\
% MuSiQue   & 0.26/0.34 & 0.24/0.30 & 0.25/0.33 & 0.30/0.39 & 0.29/0.38 & \textbf{0.36/0.43} \\
% Bamboogle & 0.33/0.42 & 0.31/0.40 & 0.35/0.43 & 0.28/0.37 & 0.28/0.36 & \textbf{0.46/0.56} \\
% \bottomrule
% \end{tabular}%
% }
% \vspace{-7mm}
% \end{table*}

\begin{table*}[t]
\centering
\caption{\textbf{Stage-1 component ablation on Qwen3-32B (\budgetuppermid{} budget).}
Each cell reports \texttt{EM/F1} under the hard dual-budget audit.
The 1st, 2nd, and 3rd best results are highlighted in \colorbox{colorfirst}{\textbf{first}}, \colorbox{colorsecond}{second}, and \colorbox{colorthird}{third} colors. Cell background colors are ranked by the average of EM and F1 scores. Best EM and F1 are independently bolded. $\uparrow$ indicates higher is better.}
\label{tab:stage1_component_ablation_qwen32b_full}
\scriptsize
\setlength{\tabcolsep}{4.2pt}
\renewcommand{\arraystretch}{1.08}
{\fontfamily{ptm}\selectfont
\resizebox{\textwidth}{!}{%
\begin{tabular}{lcccccc}
\toprule
\multirow{2}{*}{\textbf{Benchmark}} & \textbf{\vanilla{}} & \textbf{w/o penalty} & \textbf{w/o norm.} & \textbf{w/o struct.} & \textbf{w/o guards} & \textbf{\method{} (Full)} \\
\cmidrule(lr){2-7}
& \textbf{EM/F1} $\uparrow$ & \textbf{EM/F1} $\uparrow$ & \textbf{EM/F1} $\uparrow$ & \textbf{EM/F1} $\uparrow$ & \textbf{EM/F1} $\uparrow$ & \textbf{EM/F1} $\uparrow$ \\
\midrule
HotpotQA  & \cellcolor{colorsecond} 0.34/0.41 & \cellcolor{colorsecond} 0.35/0.40 & 0.34/0.38 & 0.32/0.37 & \cellcolor{colorthird} 0.34/0.39 & \cellcolor{colorfirst} \textbf{0.39}/\textbf{0.47} \\
\twowiki{}   & 0.40/0.45 & 0.38/0.43 & \cellcolor{colorsecond} 0.44/0.50 & 0.39/0.45 & \cellcolor{colorthird} 0.42/0.48 & \cellcolor{colorfirst} \textbf{0.54}/\textbf{0.63} \\
MuSiQue   & 0.26/0.34 & 0.24/0.30 & 0.25/0.33 & \cellcolor{colorsecond} 0.30/0.39 & \cellcolor{colorthird} 0.29/0.38 & \cellcolor{colorfirst} \textbf{0.36}/\textbf{0.43} \\
Bamboogle & \cellcolor{colorthird} 0.33/0.42 & 0.31/0.40 & \cellcolor{colorsecond} 0.35/0.43 & 0.28/0.37 & 0.28/0.36 & \cellcolor{colorfirst} \textbf{0.46}/\textbf{0.56} \\
\bottomrule
\end{tabular}%
}
}
\vspace{-7mm}
\end{table*}

Table~\ref{tab:stage1_component_ablation_qwen32b_full} reports the benchmark-level results. The full controller achieves the best F1 on all four benchmarks, while every component removal reduces the macro average. The largest drop comes from removing budget-dependent penalty, showing that the remaining-budget term is not a cosmetic penalty but a central part of the action-allocation rule. Removing normalization, structural signals, or guards also hurts, although the affected benchmark differs. This pattern supports the method design: the improvement is tied to the three-stage task-level VOI score controller rather than to a generic prompt constraint, search-procedure change, or answer-time finalization effect.

\paragraph{Inference-time cost.}
Table~\ref{tab:runtime_cost_ablation} reports an end-to-end wall-clock timing probe under the Qwen3-32B \budgetuppermid{} setting with 100 examples per dataset. The full \method{} controller reduces mean inference time relative to \vanilla{} on all four benchmarks, from 20.91s to 15.23s on average, a 27.2\% reduction. Thus the added deterministic controller does not introduce a visible latency burden in this setting; its action choices often shorten trajectories by avoiding low-value operations. The \twowiki{} row also shows a boundary case where some ablated variants are faster than the full controller, indicating dataset-dependent runtime behavior.

\begin{table*}[t]
\centering
\caption{\textbf{End-to-end inference time under Qwen3-32B and the \budgetuppermid{} budget.}
Each cell reports mean wall-clock seconds per example; parenthesized values give relative change against \vanilla{}.
Lower is better. The 1st, 2nd, and 3rd best results, ranked solely by mean inference time, are highlighted in
\colorbox{colorfirst}{\textbf{first}}, \colorbox{colorsecond}{second}, and \colorbox{colorthird}{third} colors.}
\label{tab:runtime_cost_ablation}
\scriptsize
\setlength{\tabcolsep}{4.2pt}
\renewcommand{\arraystretch}{1.08}
{\fontfamily{ptm}\selectfont
\resizebox{\textwidth}{!}{%
\begin{tabular}{lcccccc}
\toprule
\multirow{2}{*}{\textbf{Benchmark}} 
& \textbf{\vanilla{}}
& \textbf{w/o penalty}
& \textbf{w/o norm.}
& \textbf{w/o struct.}
& \textbf{w/o guards}
& \textbf{\method{} (Full)} \\
\cmidrule(lr){2-7}
& \textbf{Mean Time (s)} $\downarrow$
& \textbf{Mean Time (s)} $\downarrow$
& \textbf{Mean Time (s)} $\downarrow$
& \textbf{Mean Time (s)} $\downarrow$
& \textbf{Mean Time (s)} $\downarrow$
& \textbf{Mean Time (s)} $\downarrow$ \\
\midrule
\bamboogle
& 19.92
& \cellcolor{colorsecond} 16.25 {\scriptsize(-18.4\%)}
& 16.38 {\scriptsize(-17.8\%)}
& 17.15 {\scriptsize(-13.9\%)}
& \cellcolor{colorthird} 16.34 {\scriptsize(-18.0\%)}
& \cellcolor{colorfirst} \textbf{14.41} {\scriptsize(-27.7\%)} \\

\hotpot
& 22.01
& \cellcolor{colorsecond} 17.44 {\scriptsize(-20.7\%)}
& 17.79 {\scriptsize(-19.1\%)}
& 17.57 {\scriptsize(-20.1\%)}
& \cellcolor{colorthird} 17.45 {\scriptsize(-20.7\%)}
& \cellcolor{colorfirst} \textbf{14.41} {\scriptsize(-34.5\%)} \\

\musique
& 20.23
& 16.65 {\scriptsize(-17.7\%)}
& 16.98 {\scriptsize(-16.1\%)}
& \cellcolor{colorthird} 16.47 {\scriptsize(-18.6\%)}
& \cellcolor{colorsecond} 16.22 {\scriptsize(-19.8\%)}
& \cellcolor{colorfirst} \textbf{14.79} {\scriptsize(-26.9\%)} \\

\twowiki
& 21.49
& 15.81 {\scriptsize(-26.4\%)}
& \cellcolor{colorfirst} \textbf{15.25} {\scriptsize(-29.0\%)}
& \cellcolor{colorsecond} 15.53 {\scriptsize(-27.7\%)}
& \cellcolor{colorthird} 15.64 {\scriptsize(-27.2\%)}
& 17.30 {\scriptsize(-19.5\%)} \\

\midrule
\textbf{Average}
& 20.91
& \cellcolor{colorthird} 16.54 {\scriptsize(-20.9\%)}
& 16.60 {\scriptsize(-20.6\%)}
& 16.68 {\scriptsize(-20.2\%)}
& \cellcolor{colorsecond} 16.41 {\scriptsize(-21.5\%)}
& \cellcolor{colorfirst} \textbf{15.23} {\scriptsize(-27.2\%)} \\
\bottomrule
\end{tabular}}
}
\vspace{-5mm}
\end{table*}
\subsection{Two-Stage Component Ablation}
\label{subsec:two_stage_ablation}

Figure~\ref{fig:two_stage_component_ablation} isolates the two stages under the same audited setting. Stage~1 alone improves F1 on all four benchmarks, confirming that most of the gain comes from search-time budget allocation. The full method yields relative F1 gains of $+5.7\%$, $+11.8\%$, $+14.7\%$, and $+18.4\%$ over \vanilla{} on HotpotQA, 2WikiMultihopQA, MuSiQue, and Bamboogle, respectively. Stage~2 contributes no additional gain on HotpotQA, but accounts for $13.4\%$, $41.9\%$, and $27.8\%$ of the total gain on 2WikiMultihopQA, MuSiQue, and Bamboogle. This supports the intended division of labor: Stage~1 provides broad budget-aware search control, while Stage~2 acts as a sparse exactness layer when the retrieved trajectory is already mostly adequate.
The finalizer is sparse, conservative, and exactness-oriented: it avoids broad rewriting and provides targeted corrections when the retrieved trajectory is adequate but the answer still has an answer-form error.

\section{Discussion and Limitations}
\label{sec:discussion}
\begin{wrapfigure}[20]{r}{0.35\textwidth}
\vspace{-1.8em}
\centering
\includegraphics[width=\linewidth]{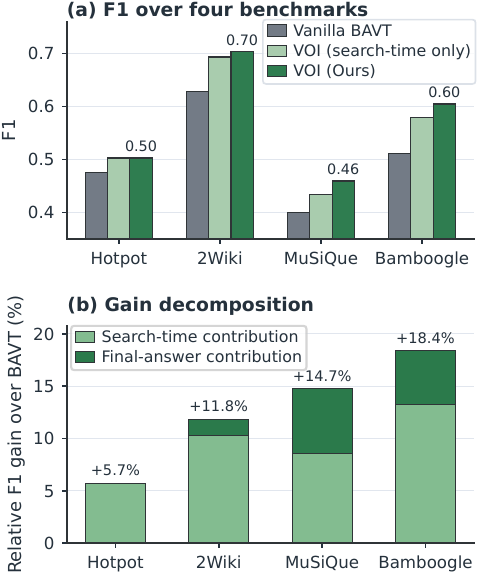}
\caption{\textbf{Two-stage component ablation.}
Top: absolute F1 for \vanilla{}, search-time-only \method{}, and full \method{}.
Bottom: relative F1 gain over \vanilla{}, split into search-time and final-answer contributions; labels report total gain and the Stage~2 share.}
\label{fig:two_stage_component_ablation}
\vspace{-7mm}
\end{wrapfigure}
Our results show that budget-aware search is not a monotone scaling problem. Larger tool-call and output-token budgets can improve evidence coverage, but may also introduce redundant search or noisier finalization. This is visible on \twowiki{} and \bamboogle{}, where middle budgets sometimes outperform \budgethigh{}. The task-level VOI controller mitigates this trade-off by making budget spending state-dependent, but does not remove the tension between exploration and commitment.

The main boundary appears in high-budget regimes and across backbones. \bats{} remains competitive in several upper-budget cells, and Qwen3.5-122B shows a more mixed regime in Appendix~\ref{app:qwen35_scaling} because stronger backbones narrow the relative gap, consistent with the diminishing-returns regime expected for any inference-time controller layered on a more capable base model. The finalizer is also exactness-oriented rather than reasoning-complete: it can repair local answer-form errors, but cannot recover from wrong retrieval paths or unresolved bridge relations. Extending the controller to stronger retrieval backends, richer query rewriting, and longer-horizon browsing remains future work.

\section{Conclusion}
\label{sec:conclusion}

This paper studies tool-augmented LLM search under explicit tool-call and output-token budgets. We formulate the problem as two-stage constrained inference: a search-time decision over how to spend the next unit of budget, followed by an answer-time decision over whether a refined answer is worth the intervention risk. Our method instantiates this formulation with a training-free controller based on the task-level VOI score over \search{}, \decompose{}, and \answer{}, together with a conservative evidence-grounded finalizer that rewrites only under low-risk exactness conditions.
Across four multi-hop QA benchmarks, four budget levels, and three LLM backbones, the results show that budget control is useful but not uniform. The search-time controller provides the main gains by improving action allocation under strict budgets, while the finalizer adds a sparse exactness benefit when the trajectory already contains adequate evidence. The remaining failures clarify the boundary of the approach: more budget is not always better, backbone behavior matters, and answer-time control cannot repair incorrect retrieval paths.

\textbf{Broader impacts.}
Budget-aware search can improve the deployability of tool-augmented agents by reducing unnecessary tool use, making inference cost more predictable, and exposing when an agent chooses to search, decompose, or answer. These properties are useful for applications where latency, token usage, or external tool calls must be controlled. At the same time, more efficient search agents could also be used in harmful information-seeking workflows. We therefore emphasize hard budget audits, explicit accounting, conservative finalization, and failure analysis, so that budgeted agent behavior is easier to inspect rather than easier to hide.

% \paragraph{Broader impact.}
% Budget-aware search can improve deployability, transparency, and cost control for information-seeking agents, especially when tool calls or generated tokens are expensive. The same techniques could also be used to optimize harmful agents under fixed resource limits. We therefore emphasize hard budget audits, explicit failure cases, and conservative finalization so that budgeted agent behavior is easier to inspect rather than easier to hide.
\bibliographystyle{refstyle}
\bibliography{references}

\newpage
\appendix

\section{Controller Behavior Analysis}
\label{app:mechanism}

This appendix complements Section~\ref{sec:method} with descriptive evidence on the realized behavior of the search-time controller. These results are not additional algorithmic components; they visualize how the implemented controller scores feasible actions and activates guards across budget and compositionality states.

\begin{figure*}[hbpt]
\centering
\includegraphics[width=0.98\textwidth]{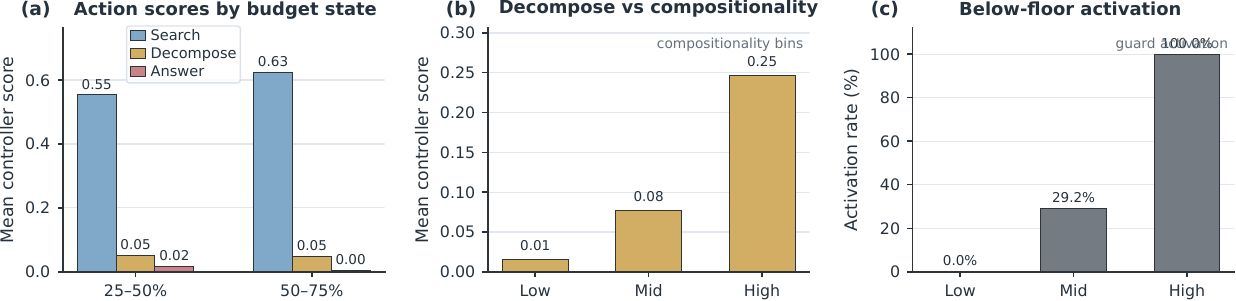}
\caption{\textbf{Empirical behavior of the search-time controller.}
\textbf{(a)} Mean controller scores for \search{}, \decompose{}, and \answer{} across realized remaining-budget bands.
\textbf{(b)} The mean \decompose{} score increases with question compositionality, consistent with the intended role of decomposition in resolving bridge structure.
\textbf{(c)} The below-floor guard activates mainly for high-compositionality states, reflecting the conservative minimum-search behavior used to avoid premature answer commitment.
All quantities are descriptive statistics computed from realized audited trajectories.}
\label{fig:controller_mechanism_diagnostics}
\vspace{-0.5em}
\end{figure*}
\begin{table}[hbpt]
\centering
\small
\caption{\textbf{Summary of the search-time controller.}}
\label{tab:controller_summary}
\begin{tabular}{p{0.17\linewidth}p{0.33\linewidth}p{0.33\linewidth}}
\toprule
\textbf{Action} & \textbf{Main positive signals} & \textbf{Main suppressing signals} \\
\midrule
\textsc{Search} &
unresolved evidence, missing new support, low closure, loop pressure &
high budget pressure, repeated search saturation \\
\textsc{Decompose} &
high compositionality, unresolved bridge structure, stagnation after search &
factoid-like question, repeated decomposition saturation \\
\textsc{Answer} &
high closure, strong answer support, candidate answer present, tighter budget &
early-answer risk, unresolved bridge structure, weak support \\
\bottomrule
\end{tabular}
\end{table}

As summarized in Table~\ref{tab:controller_summary}, the search-time controller is a fixed rule-based scorer over three actions: \textsc{Search}, \textsc{Decompose}, and \textsc{Answer}. At each decision step, it evaluates four aspects of the current search state. First, it estimates whether important evidence remains unresolved and whether another retrieval step is still likely to help. Second, it checks whether the question appears structurally compositional, for example through unresolved bridge structure, in which case decomposition becomes more valuable than ordinary search. Third, it evaluates answer readiness through closure, answer support, and the presence of a candidate answer span, while penalizing premature commitment when support remains weak. Fourth, it applies budget pressure and action-specific costs so that expensive actions become less attractive as the remaining tool and token budget shrinks. In addition to these smooth score terms, the controller uses a small number of hard guards: it enforces a minimum amount of search on compositional cases, suppresses unnecessary decomposition on factoid-like questions, and blocks overly early answer commitment when structural risk remains high.

\section{Inference Algorithm and Prompt Modules}
\label{app:algorithms}
\nolinenumbers

\begin{algorithm}[t]
\caption{Two-Stage Budget-Aware Inference}
\label{alg:two_stage_inference}
\footnotesize
\begin{algorithmic}[1]
\State \textbf{Input:} question $x$, tool budget $B_{\mathrm{tool}}$, output-token budget $B_{\mathrm{tok}}$
\State initialize the search procedure, frontier $\mathcal{F}$, trace $\mathcal{T}$, and remaining budget $b_0=(B_{\mathrm{tool}},B_{\mathrm{tok}})$
\While{$\mathcal{F}\neq\emptyset$ and $b_t$ is not exhausted and termination has not fired}
    \State select the current frontier state $s_t$ according to the search procedure
    \State construct the feasible operation set $\mathcal{K}^{\mathrm{feas}}_t$
    \For{each feasible action $k\in\mathcal{K}^{\mathrm{feas}}_t$}
        \State compute action-specific costs and controller features from the current search state
    \State score $k$ through the three-stage pipeline: task-level action utility $u_t(k)$ (Eq.~\eqref{eq:controller_score}), normalized raw score $r_t(k)$ (Eq.~\eqref{eq:normalized_score}), guarded executable score $\widetilde{\mathcal{J}}_t(k)$ (Eq.~\eqref{eq:guarded_score})
    \EndFor
    \State select the highest-scoring feasible action $k_t$ and execute it
    \State update the trace $\mathcal{T}$, the frontier $\mathcal{F}$, and the remaining budget $b_{t+1}$
\EndWhile
\State extract the base answer $a_{\mathrm{base}}$ from the best available trajectory
\State construct the refined candidate $a_{\mathrm{ref}}$ from the completed trace $\mathcal{T}$
\State build $z=z(\mathcal{T},a_{\mathrm{base}},a_{\mathrm{ref}})$ and apply the finalization rule (Eq.~\eqref{eq:finalizer_score})
\If{$z\in\mathcal{S}_{\mathrm{safe}}$ and $F(z)\ge\tau$}
    \State \Return $a_{\mathrm{ref}}$
\Else
    \State \Return $a_{\mathrm{base}}$
\EndIf
\end{algorithmic}
\end{algorithm}

\subsection{System Prompt}

Our method reuses the same single-backbone search procedure as BAVT for planning, generation, and value estimation. The central difference is not that we replace the generator with a new workflow, but that we intervene \emph{before} generation with an explicit search-time controller. Concretely, the controller scores feasible actions through the three-stage pipeline in Eqs.~\eqref{eq:controller_score}--\eqref{eq:guarded_score} and selects
\[
k_t^\star = \arg\max_{k \in \mathcal{K}^{\mathrm{feas}}_t} \widetilde{\mathcal{J}}_t(k),
\]
and only then materializes the selected action as a deterministic prompt insertion. The generator therefore still follows a standard system-prompt plus user-turn format, but its behavior is constrained by a controller-selected instruction whose choice is determined by the task-level VOI score outside the LLM. This is the key interface between the theory and the prompt design: the theory decides \emph{which} action should spend the next unit of budget, and the prompt enforces \emph{how} that chosen action must be executed.

\AppendixPromptTitle{System Prompt: Budget-Aware Generator with VOI-Guided Control}
\begin{lstlisting}[style=appendixpromptstyle]
You are a precise AI assistant for budget-constrained multi-hop question answering.

Rules:
- Think step by step and explicitly show your reasoning.
- In each turn, do exactly one action: either call one tool or provide the final answer.
- You may use tools multiple times across turns if needed.
- Follow the current dynamic instruction exactly; it defines what kind of action is allowed in this turn.
- If the instruction requires retrieval, you must produce exactly one <tool_call>.
- If the instruction requires answering, you must not call tools.
- Keep the final answer to a single short grounded span, not a full sentence.
- If an entity may be ambiguous, disambiguate the query with a type or role already supported by the question or evidence.

<budget>
Tool Budget Used: {tool_budget_used}/{tool_budget_total}
Tool Budget Remaining: {b_tool}
Output Token Budget Used: {token_budget_used}/{token_budget_total}
Output Token Budget Remaining: {b_token}
</budget>

{plan_context}

{dynamic_instruction}
\end{lstlisting}

\subsection{Planning Module}

As in BAVT, the agent maintains an abstract plan before and during search. The role of the planning module is not to guess the answer, but to identify the missing pieces of evidence, keep track of bridge facts, and prevent the search from revisiting already exhausted directions. The new contribution of our method is not the planner itself, but the fact that the remaining budget explicitly changes which kind of step the agent is allowed to take next.

\AppendixPromptTitle{Planning Module: High-Level Search Plan}
\begin{lstlisting}[style=appendixpromptstyle]
Before taking stepwise actions, write a brief high-level plan.

- Identify the entities, relations, or attributes that still need evidence.
- Separate direct answer clues from bridge facts that require an intermediate hop.
- Keep the plan abstract; do not invent facts.
- Update the plan only when new evidence changes the remaining uncertainty.
\end{lstlisting}

\subsection{Stage 1: Search-Time Budget Allocation}

Stage 1 is the main search-time contribution of our method. Instead of relying only on BAVT's widen-versus-deepen routing, we explicitly allocate the next unit of budget among three action types: \search, \decompose, and \answer. The action decision is made by the three-stage pipeline based on the task-level VOI score in Eqs.~\eqref{eq:controller_score}--\eqref{eq:guarded_score}; the prompt is only the execution layer that instantiates this decision inside the generator. In other words, the prompt does not choose the action; the controller computes the task-level VOI score. The controller computes a deterministic score from explicit trajectory features, fixed coefficients, budget-aware action costs, and conservative guards, and then injects the chosen action as a deterministic instruction.

\AppendixPromptTitle{User Turn Format for Stage 1}
\begin{lstlisting}[style=appendixpromptstyle]
{history}

Continue reasoning.
If evidence is not yet sufficient, call exactly one search tool.
If evidence is sufficient, answer now.

Output format:
<thought>your reasoning</thought>
<tool_call>search query</tool_call>
<answer>FINAL_ANSWER</answer>

Rules for this turn:
- Emit either <tool_call> or <answer>, not both, unless the dynamic instruction explicitly allows answering now.
- When you use <answer>, include only the shortest correct answer span and no explanation.
- When you use <tool_call>, the query must be concrete, grounded in the question, and non-redundant with the immediately preceding failed path.
- If the question asks for a specific place, date, person, title, or slot value, the answer must use that grounded string rather than a broader paraphrase.
- If the current hypothesis is ambiguous, disambiguate the next query explicitly with a type or role.
\end{lstlisting}

The controller-to-prompt interface is therefore
\[
(s_t,b_t)
\xrightarrow{\text{Eqs.~\eqref{eq:controller_score}--\eqref{eq:guarded_score}}}
k_t^\star\in\{\search,\decompose,\answer\}
\xrightarrow{\text{prompt injection}}
\texttt{dynamic\_instruction}.
\]
The four action templates below are the paper-specific prompt components that implement this interface.

\AppendixPromptTitle{Stage 1 Instruction: Search}
\begin{lstlisting}[style=appendixpromptstyle]
Instruction: SEARCH.
The expected marginal utility of one more retrieval step is highest.
You must call exactly one search tool now.
Do not answer and do not continue with tool-free reasoning.

Inside <thought>, do all of the following:
- State the most specific unresolved fact that still blocks the answer.
- Explain briefly why another retrieval step is preferable to answering now.
- Reuse concrete entities already grounded by the question or evidence.
- If there is ambiguity, state how the query will disambiguate it.

Then emit exactly one <tool_call>.
The query should be targeted, entity-specific, and aimed at closing one missing factual gap rather than broad exploration.
Do not output <answer> in this turn.
\end{lstlisting}

\AppendixPromptTitle{Stage 1 Instruction: Decompose}
\begin{lstlisting}[style=appendixpromptstyle]
Instruction: DECOMPOSE.
You have been searching without enough progress, and the remaining uncertainty appears to be compositional.
You must call exactly one search tool now.
Do not answer and do not continue with tool-free reasoning.

Inside <thought>, do all of the following:
- Identify the missing bridge entity, intermediate fact, or latent relation that connects the hops.
- Summarize what has already been established from the current trajectory.
- Explain why another ordinary search step is less useful than an explicit bridge-fact query.
- Formulate one precise retrieval target that would resolve the compositional gap.

Then emit exactly one <tool_call> with that targeted query.
The query must focus on the missing bridge information, not on re-searching the full question from scratch.
Do not output <answer> in this turn.
\end{lstlisting}

\AppendixPromptTitle{Stage 1 Instruction: Answer}
\begin{lstlisting}[style=appendixpromptstyle]
Instruction: ANSWER.
The controller judges that answer commitment has higher utility-per-cost than further retrieval.
Do not call tools.

Inside <thought>, briefly verify that the answer is directly supported by the evidence already gathered.
Then return exactly one line in the form <answer>FINAL_ANSWER</answer>.

Answer requirements:
- The answer must be the shortest grounded answer span.
- Do not output explanation outside the tag.
- If the question is yes/no, FINAL_ANSWER must be exactly yes or no.
- If the question is a binary choice, output only the selected option.
\end{lstlisting}

\AppendixPromptTitle{Stage 1 Instruction: Budget Backstop}
\begin{lstlisting}[style=appendixpromptstyle]
Instruction: BUDGET BACKSTOP.
The remaining budget is exhausted or too small for another meaningful retrieval step.
You must answer now and you must not call tools.

Inside <thought>, use only the evidence already present in the trajectory.
Do not speculate beyond the grounded information.
If the evidence is partial, prefer the most directly supported short answer rather than an expansive paraphrase.

Return exactly one line in the form <answer>FINAL_ANSWER</answer>.
Keep the answer as short as possible.
If the question is yes/no, FINAL_ANSWER must be exactly yes or no.
\end{lstlisting}

Table~\ref{tab:stage1_budget_allocation} summarizes the Stage-1 interface. The key distinction from vanilla BAVT is that our controller makes an explicit budget-allocation decision over action types before generation. In the released code, we also include simple guardrails to prevent tool-free stalling and to enforce a required search step whenever the controller allocates budget to retrieval, but we do not treat these guardrails as separate method components.

\begin{table}[H]
\centering
\scriptsize
\caption{\textbf{Stage 1 compared with related prompting styles.} Our method differs from prior systems not by adding more free-form prompt complexity, but by using an explicit controller to choose which prompt constraint is injected at each step.}
\label{tab:stage1_budget_allocation}
\setlength{\tabcolsep}{3pt}
\renewcommand{\arraystretch}{1.12}
\begin{tabular}{@{}P{0.21\linewidth}C{0.17\linewidth}C{0.18\linewidth}C{0.20\linewidth}C{0.17\linewidth}@{}}
\toprule
\textbf{Method} & \textbf{\shortstack{Interleaved\\retrieval}} & \textbf{\shortstack{Step-level\\instruction}} & \textbf{\shortstack{Budget-aware\\action choice}} & \textbf{\shortstack{Conservative\\finalization}} \\
\midrule
\searcho & \cmark & Partial & \xmark & \xmark \\
\vanilla & \cmark & \cmark & Partial & \xmark \\
\textbf{Ours} & \cmark & \cmark & \cmark & \cmark \\
\bottomrule
\end{tabular}
\end{table}

Here ``budget-aware action choice'' means that the action type is selected before generation by an explicit controller rather than inferred from free-form continuation. In our case, the controller first selects \textsc{Search}, \textsc{Decompose}, or \textsc{Answer}; the corresponding instruction is then injected into the generator.

\subsection{Stage 2: Answer-Time Finalization}

Stage 2 is an answer-selection module applied after search terminates. Starting from the completed trajectory, the system compares the trajectory answer with a refined candidate derived from the same evidence. The goal is not to launch another round of free-form rewriting, but to correct local exactness errors when the evidence clearly supports a safer, more specific answer span. If the case still looks structurally risky, such as unresolved bridge reasoning or comparative semantics, the module abstains and keeps the trajectory answer. Importantly, this stage adds no new tool calls and no additional LLM call during finalization.

We therefore present Stage 2 as an answer-selection card rather than as another generator prompt. The point is precisely that this module is \emph{not} a fresh LLM rewrite call. It is a conservative selection policy over two candidates already available from the completed trace.

\AppendixPromptTitle{Stage 2 Selection Card: Conservative Answer Finalization}
\begin{lstlisting}[style=appendixpromptstyle]
Inputs:
- Question
- Trajectory answer $a_{\mathrm{base}}$
- Refined candidate $a_{\mathrm{ref}}$
- Supporting evidence from the completed trajectory

Decision principle:
- If the remaining error is only local answer exactness, prefer the better-supported candidate.
- If the case still contains bridge ambiguity, comparative semantics, or path-level uncertainty, abstain.
- Never add a new tool call.
- Never launch an additional LLM call during finalization.
- Treat abstention as the default whenever the gain-risk trade-off is unclear.

Typical positive cases:
- yes/no polarity repair
- binary choice repair
- typed-slot correction
- supported factoid completion

Typical abstention cases:
- unresolved bridge reasoning
- comparative questions
- decomposition-heavy trajectories
- cases where rewriting would likely change semantics rather than improve exactness
\end{lstlisting}

Table~\ref{tab:answer_time_policy} summarizes this Stage-2 policy. The organizing principle is conservative answer finalization: intervene only when the remaining error appears local to answer-form exactness, and keep the trajectory answer whenever the evidence suggests the real failure is path-level.

\begin{table*}[hbpt]
\centering
\small
\caption{\textbf{Stage 2: Answer-time finalization policy.} This module is conservative by design: it rewrites only when the refined candidate is better supported and the risk of changing the semantics remains low.}
\label{tab:answer_time_policy}
\begin{tabular}{P{0.19\textwidth}P{0.71\textwidth}}
\toprule
\textbf{Case type} & \textbf{Finalization behavior} \\
\midrule
Unresolved bridge or comparative reasoning & Abstain. If the trajectory still reflects unresolved bridge structure, comparison logic, or other path-level uncertainty, keep the trajectory answer. \\
Yes/no or binary choice & Finalize when the evidence clearly resolves the polarity or the choice between the listed options. This is the cleanest type of answer-time repair. \\
Slot-filling exactness & Finalize only when the refined answer adds a small but meaningful amount of specificity, such as a capacity, date, or year range, without changing the underlying fact. \\
Supported factoid completion & Finalize when the refined candidate simply completes an already supported fact span, for example by restoring a canonical phrase or an omitted attribute. \\
General fallback & Otherwise finalize only when the refined candidate is better supported by the trajectory evidence and remains comparably concise; abstain in all remaining cases. \\
\bottomrule
\end{tabular}
\end{table*}

This policy is the appendix-level counterpart of Eq.~\eqref{eq:finalizer_score}. In words, the answer-time module behaves like a high-precision selector rather than a generic editor: it prefers abstention whenever the risk of disturbing a correct trajectory answer outweighs the likely exactness gain.

\subsection{Representative Cases}

Instead of showing long traces, we summarize representative Stage-2 behaviors as compact decision cards. The first three cards are positive exactness repairs, and the last card is an abstention example showing when the finalizer deliberately refuses to rewrite.

\AppendixPromptTitle{Case Card 1: Binary-Choice Exactness Repair}
\begin{lstlisting}[style=appendixpromptstyle]
Question:
Which writer was from England, Henry Roth or Robert Erskine Childers?

Trajectory answer:
Robert Erskine Childers

Refined candidate:
Robert Erskine Childers DSC

Evidence:
The supporting evidence identifies Childers as the English writer and also supports the fuller canonical name span.

Why Stage 2 finalizes:
- This is a binary-choice question rather than an open-ended generation case.
- The refinement does not change which option is selected.
- The refined string is better supported by the trajectory evidence.
- The residual error is local answer exactness, not path-level reasoning.

Outcome:
Finalize to the fuller supported answer span.
\end{lstlisting}

\AppendixPromptTitle{Case Card 2: Supported Factoid Completion}
\begin{lstlisting}[style=appendixpromptstyle]
Question:
What distinction is held by the former NBA player who was a member of the Charlotte Hornets during their 1992--93 season and was head coach for the Charlotte Sting?

Trajectory answer:
shortest player ever to play in the NBA

Refined candidate:
shortest player ever to play in the National Basketball Association

Evidence:
The trajectory evidence ties Muggsy Bogues to both team clues and supports the fuller canonical statement of the distinction.

Why Stage 2 finalizes:
- The refined candidate is not a different claim.
- It is a completion of the same already-supported fact.
- The intervention improves exactness while preserving semantics.
- The gain is local and the rewrite risk is low.

Outcome:
Finalize to the fuller supported factoid span.
\end{lstlisting}

\AppendixPromptTitle{Case Card 3: Typed-Slot Exactness Repair}
\begin{lstlisting}[style=appendixpromptstyle]
Question:
The arena where the Lewiston Maineiacs played their home games can seat how many people?

Trajectory answer:
3,677

Refined candidate:
3,677 seated

Evidence:
The supporting passage contains the exact capacity-bearing phrase rather than the bare number alone.

Why Stage 2 finalizes:
- This is a slot-value correction problem rather than a search failure.
- The refined candidate adds only minimal slot-bearing wording.
- The underlying fact is unchanged.
- The support for the refined phrase is more direct than for the shorter numeric span alone.

Outcome:
Finalize to the typed-slot correction.
\end{lstlisting}

\AppendixPromptTitle{Case Card 4: Abstention under Bridge Risk}
\begin{lstlisting}[style=appendixpromptstyle]
Question:
Which performance act has a higher instrument to person ratio, Badly Drawn Boy or Wolf Alice?

Trajectory answer:
Badly Drawn Boy

Refined candidate:
Badly Drawn Boy

Evidence:
The trajectory already supports the current answer, but the question still has comparative and bridge-sensitive structure.

Why Stage 2 abstains:
- The remaining risk is in comparative multi-hop reasoning, not answer wording.
- A rewrite would offer little exactness gain.
- In this regime, conservative abstention is safer than intervention.
- The finalizer therefore preserves the trajectory answer instead of forcing an unnecessary change.

Outcome:
Abstain and keep the trajectory answer.
\end{lstlisting}

These cards illustrate the intended role of answer-time control. Positive interventions repair answer-form errors that remain after the search trajectory is already essentially correct. The abstention example shows the opposite boundary: when the residual uncertainty is still about multi-hop structure rather than wording, Stage 2 keeps the original trajectory answer.

\section{Theoretical Statements and Proofs}
\label{app:proofs}
\linenumbers

This appendix provides the formal analysis summarized in Section~\ref{sec:method}. We first state the main theorem-level results, then give the definitions, assumptions, auxiliary lemmas, and proofs. The analysis supports the two controller layers: search-time budget allocation and answer-time finalization. It is local and decision-level; it does not claim global optimality of the full search tree.

\subsection{Main Theoretical Statements}
\label{app:theory_statements}

For search-time control, fix step $t$, state $s_t$, budget $b_t$, and feasible action set $\mathcal{K}_t^{\mathrm{feas}}$. Let $V_{t+1}(s,b)$ be the analysis-only oracle continuation value: the scalar maximum expected final reward attainable from state $s$ with remaining budget $b$ under the oracle continuation policy. Define
\[
Q_t^\star(k):=\mathbb{E}\!\left[V_{t+1}\!\bigl(S_{t+1}^{(k)},\,b_t-g_t(k)\bigr)\mid s_t,b_t\right]
\]
as the oracle one-step lookahead value, where $S_{t+1}^{(k)}$ is the next state and $g_t(k)$ is the budget charge. With $V_t^{\mathrm{stop}}$ denoting immediate-stop value, define $\widetilde{Q}_t^\star(k):=Q_t^\star(k)-V_t^{\mathrm{stop}}(s_t,b_t)$ and decompose
\[
\widetilde{Q}_t^\star(k)=\Delta_t^\star(k)+\Psi_t^\star(k)-\Lambda_t^\star(k;b_t)+\xi_t(k),
\]
where $\Delta_t^\star(k)$ is the oracle immediate-gain term, $\Psi_t^\star(k)$ the structural residual, $\Lambda_t^\star(k;b_t)$ the local oracle budget shadow-cost term, and $\xi_t(k)$ a smoothness remainder. This decomposition is only an analysis device mirroring Eq.~\eqref{eq:controller_score}. Let $k_t$ be the implemented action and $k_t^{\mathrm{opt}}$ an oracle-best feasible action. Define
\[
\Gamma_t(k):=\varepsilon_t^{\Delta}+\varepsilon_t^{\Psi}+\varepsilon_t^{\Pi}+\beta_t\|g_t(k)\|+\tfrac{L_t}{2}\|g_t(k)\|^2,
\]
where $\varepsilon_t^{\Delta},\varepsilon_t^{\Psi},\varepsilon_t^{\Pi}$ bound gain, structural, and budget-penalty approximation errors, $\beta_t$ bounds shadow-price mismatch, and $L_t$ is the smoothness constant. The oracle counterpart of the task-level VOI score in Eq.~\eqref{eq:normalized_score} is $r_t^\star(k)=[U_t^\star(k;b_t)]_+/(d_t(k;b_t)+\epsilon)$, where $U_t^\star$ is defined below.

\begin{theorem}[Local approximation of the task-level action utility]
\label{thm:controller_voi_local}
Under the budget smoothness, shadow-price homogeneity, and component-wise approximation conditions in Assumptions~\ref{assump:budget_smooth}--\ref{assump:score_approx}, every feasible action $k\in\mathcal{K}_t^{\mathrm{feas}}$ satisfies
\begin{equation}
\left|u_t(k)-\widetilde{Q}_t^\star(k)\right|\le \Gamma_t(k).
\label{eq:local_utility_bound}
\end{equation}
\end{theorem}

Theorem~\ref{thm:controller_voi_local} covers the task-level action utility $u_t(k)$ and the normalized task-level VOI score $r_t(k)$ before guards. Corollary~\ref{cor:ranking_consistency} below gives a margin condition under which the normalized task-level VOI score ranking is preserved before guards, and Corollary~\ref{cor:value_gap} gives a one-step value-gap bound when the guard layer is compatible with the utility ranking. The controller should be reliable when the score clearly separates feasible actions, while failures are expected when approximation terms in $\Gamma_t(k)$ are large or guards intentionally override raw ranking.

For answer-time finalization, let $\mathcal{Z}$ be the range of $z$, let $Z:=z(\mathcal{T},a_{\mathrm{base}},a_{\mathrm{ref}})$, and let $\Delta_{\mathrm{fin}}:=R(a_{\mathrm{ref}},y)-R(a_{\mathrm{base}},y)$ be the score change from accepting the refined answer. With $(v)_+=\max\{v,0\}$, define
\[
G^\star(z):=\mathbb{E}[(\Delta_{\mathrm{fin}})_+\mid Z=z],
\qquad
H^\star(z):=\mathbb{E}[(-\Delta_{\mathrm{fin}})_+\mid Z=z]
\]
as oracle gain and harm. We consider policies $\pi:\mathcal{Z}\to\{0,1\}$, where $\pi(z)=1$ accepts $a_{\mathrm{ref}}$; safe policies satisfy $\pi(z)=0$ for $z\notin\mathcal{S}_{\mathrm{safe}}$ and the harm constraint in Eq.~\eqref{eq:two_stage_problem}.

\begin{theorem}[Optimal threshold form of safe answer replacement]
\label{thm:finalizer_threshold}
Under the strong-duality and multiplier-attainment condition in Assumption~\ref{assump:dual}, there exists $\eta^\star\ge 1$ such that the policy
\begin{equation}
\pi^\star(z)=\mathbf{1}\!\left\{z\in\mathcal{S}_{\mathrm{safe}}
\text{ and } G^\star(z)-\eta^\star H^\star(z)\ge 0\right\}
\label{eq:optimal_safe_writeback}
\end{equation}
maximizes $\mathbb{E}[R(\hat a,y)]$ subject to $\mathbb{E}[L_{\mathrm{harm}}(\hat a,a_{\mathrm{base}},y)]\le\rho_{\mathrm{harm}}$ and $\pi(z)=0$ for $z\notin\mathcal{S}_{\mathrm{safe}}$.
\end{theorem}

\subsection{Search-Time Controller: Definitions, Assumptions, and Proof}
\label{app:search_time_proofs}

Fix step $t$, frontier state $s_t$, remaining budget $b_t$, and feasible action set $\mathcal{K}_t^{\mathrm{feas}}$. For each feasible action $k$, let $S_{t+1}^{(k)}$ be the random next state after executing $k$ and let $g_t(k)$ be its budget charge, which is $(s_t,b_t)$-measurable. For token costs, this means using the conditional expectation of realized tokens given $(s_t,b_t,k)$ and absorbing residual variance into the remainder. Let $V_{t+1}(s,b)$ be the analysis-only oracle continuation value, i.e., the scalar maximum expected final reward from state $s$ and remaining budget $b$ under the oracle continuation policy and induced future-state/evaluation distribution. The oracle one-step lookahead value is
\begin{equation}
Q_t^\star(k):=\mathbb{E}\!\left[V_{t+1}\!\bigl(S_{t+1}^{(k)},\,b_t-g_t(k)\bigr)\mid s_t,b_t\right],
\label{eq:oracle_lookahead}
\end{equation}
and $k_t^{\mathrm{opt}}\in\arg\max_{k\in\mathcal{K}_t^{\mathrm{feas}}}Q_t^\star(k)$ denotes an oracle-optimal feasible action. Define
\begin{equation}
V_t^{\mathrm{stop}}(s_t,b_t):=\mathbb{E}\!\left[R(a_t^{\mathrm{stop}},y)\mid s_t,b_t\right],
\label{eq:stop_value}
\end{equation}
\begin{equation}
\Delta_t^\star(k):=\mathbb{E}\!\left[R(a_t^{\mathrm{stop},(k)},y)-R(a_t^{\mathrm{stop}},y)\mid s_t,b_t\right],
\label{eq:true_delta}
\end{equation}
\begin{equation}
\Psi_t^\star(k):=\mathbb{E}\!\left[V_{t+1}\!\bigl(S_{t+1}^{(k)},\,b_t\bigr)\mid s_t,b_t\right]-V_t^{\mathrm{stop}}(s_t,b_t)-\Delta_t^\star(k),
\label{eq:psi_star}
\end{equation}
where $a_t^{\mathrm{stop}}$ is the answer if the trajectory is finalized immediately and $a_t^{\mathrm{stop},(k)}$ is the answer after revealing $S_{t+1}^{(k)}$ and applying the fixed immediate-finalization backstop policy at the pre-charge budget $b_t$. The budget-adjusted oracle one-step gain is $\widetilde{Q}_t^\star(k):=Q_t^\star(k)-V_t^{\mathrm{stop}}(s_t,b_t)$.

We further define the remaining-budget shaped oracle utility
\begin{equation}
U_t^\star(k;b_t) := \Delta_t^\star(k) + \Psi_t^\star(k) - \Lambda_t^\star(k;b_t),
\label{eq:oracle_utility}
\end{equation}
where $\Lambda_t^\star(k;b_t):=\lambda_t^{\star\top}g_t(k)$ is the local oracle budget shadow-cost term. The implemented budget-penalty term $\Pi_t(k;b_t)$ may be signed, for example to encourage \answer{} near exhaustion; its deviation from this oracle shadow-cost is absorbed by $\varepsilon_t^\Pi$ in Assumption~\ref{assump:score_approx}. The remainder $\xi_t(k):=\widetilde{Q}_t^\star(k)-U_t^\star(k;b_t)$ is bounded by the smoothness constants below.

\begin{assumption}[Budget smoothness and cost measurability]
\label{assump:budget_smooth}
Fix a norm $\|\cdot\|$ on $\mathbb{R}^2$ and let $\|\cdot\|_\ast$ denote its dual norm. There exists $L_t\ge 0$ such that, for every feasible action $k\in\mathcal{K}_t^{\mathrm{feas}}$: (i) the map $b\mapsto V_{t+1}(S_{t+1}^{(k)},b)$ is almost surely twice continuously differentiable near $b_t$ with $L_t$-Lipschitz gradient, so the Taylor remainder
\begin{equation}
\varrho_t(k):=V_{t+1}\!\bigl(S_{t+1}^{(k)},\,b_t-g_t(k)\bigr)-V_{t+1}\!\bigl(S_{t+1}^{(k)},\,b_t\bigr)+\nabla_b V_{t+1}\!\bigl(S_{t+1}^{(k)},\,b_t\bigr)^\top g_t(k)
\label{eq:budget_smoothness}
\end{equation}
satisfies $|\varrho_t(k)|\le \frac{L_t}{2}\|g_t(k)\|^2$ almost surely; and (ii) $g_t(k)$ is $(s_t,b_t)$-measurable.
\end{assumption}

\begin{assumption}[Local shadow-price homogeneity]
\label{assump:shadow_price}
There exist $\lambda_t^\star\in\mathbb{R}_+^2$ and $\beta_t\ge 0$ such that, for every feasible action $k\in\mathcal{K}_t^{\mathrm{feas}}$,
\begin{equation}
\left\|\mathbb{E}\!\left[\nabla_b V_{t+1}\!\bigl(S_{t+1}^{(k)},\,b_t\bigr)\mid s_t,b_t\right]-\lambda_t^\star\right\|_\ast\le \beta_t.
\label{eq:shadow_price_homogeneity}
\end{equation}
\end{assumption}

\begin{assumption}[Component-wise approximation]
\label{assump:score_approx}
There exist $\varepsilon_t^{\Delta},\varepsilon_t^{\Psi},\varepsilon_t^{\Pi}\ge 0$ such that, for every $k\in\mathcal{K}_t^{\mathrm{feas}}$,
\begin{equation}
\bigl|\widehat{\Delta}_t(k)-\Delta_t^\star(k)\bigr|\le \varepsilon_t^{\Delta}, \quad
\bigl|\Psi_t(k)-\Psi_t^\star(k)\bigr|\le \varepsilon_t^{\Psi}, \quad
\bigl|\Pi_t(k;b_t)-\Lambda_t^\star(k;b_t)\bigr|\le \varepsilon_t^{\Pi}.
\label{eq:component_approx}
\end{equation}
\end{assumption}

\begin{assumption}[Ranking-preserving guard condition]
\label{assump:guard_compat}
The guard operator $\mathfrak{G}_t$ preserves the relevant $u_t$ ranking in the following sense: the implemented action $k_t\in\arg\max_{k\in\mathcal{K}_t^{\mathrm{feas}}}\widetilde{\mathcal{J}}_t(k)$ satisfies $u_t(k_t)\ge u_t(k_t^{\mathrm{opt}})$, where $k_t^{\mathrm{opt}}\in\arg\max_{k\in\mathcal{K}_t^{\mathrm{feas}}}Q_t^\star(k)$. Equivalently, the guard layer does not promote an action with strictly lower $u_t$ score above the oracle-optimal action.
\end{assumption}

\begin{lemma}[Budget linearization]
\label{lem:budget_linearization}
Under Assumptions~\ref{assump:budget_smooth} and~\ref{assump:shadow_price}, for every $k\in\mathcal{K}_t^{\mathrm{feas}}$,
\[
\widetilde{Q}_t^\star(k) = \Delta_t^\star(k)+\Psi_t^\star(k)-\lambda_t^{\star\top}g_t(k)+\xi_t(k),
\]
where $|\xi_t(k)|\le \beta_t\|g_t(k)\|+\frac{L_t}{2}\|g_t(k)\|^2$.
\end{lemma}

\begin{proof}
Write $S:=S_{t+1}^{(k)}$, $g:=g_t(k)$, $V:=V_{t+1}$. By Assumption~\ref{assump:budget_smooth}(i),
\[
V(S,b_t-g)=V(S,b_t)-\nabla_b V(S,b_t)^\top g+\varrho_t(k), \quad |\varrho_t(k)|\le \tfrac{L_t}{2}\|g\|^2.
\]
Taking conditional expectation and using Assumption~\ref{assump:budget_smooth}(ii), so $g$ pulls out of the expectation,
\[
Q_t^\star(k)=\mathbb{E}[V(S,b_t)\mid s_t,b_t]-\mathbb{E}[\nabla_b V(S,b_t)\mid s_t,b_t]^\top g+\bar\varrho_t(k),
\]
where $|\bar\varrho_t(k)|\le \frac{L_t}{2}\|g\|^2$. By definition of $\Psi_t^\star(k)$, $\mathbb{E}[V(S,b_t)\mid s_t,b_t]=V_t^{\mathrm{stop}}+\Delta_t^\star(k)+\Psi_t^\star(k)$. Subtracting $V_t^{\mathrm{stop}}$ and writing $\xi_t(k):=(\lambda_t^\star-\mathbb{E}[\nabla_b V(S,b_t)\mid s_t,b_t])^\top g+\bar\varrho_t(k)$ gives the identity. The bound follows from H\"{o}lder's inequality and Assumption~\ref{assump:shadow_price}.
\end{proof}

\subsection{Proof of Theorem~\ref{thm:controller_voi_local} and Corollary~\ref{cor:value_gap}}
\label{app:search_time_main_proofs}

\begin{proof}[Proof of Theorem~\ref{thm:controller_voi_local}]
Fix $k\in\mathcal{K}_t^{\mathrm{feas}}$. By definition of $u_t(k)$ in Eq.~\eqref{eq:controller_score} and $U_t^\star(k;b_t)$ in Eq.~\eqref{eq:oracle_utility},
\[
u_t(k)-U_t^\star(k;b_t)
=\bigl(\widehat{\Delta}_t(k)-\Delta_t^\star(k)\bigr)+\bigl(\Psi_t(k)-\Psi_t^\star(k)\bigr)-\bigl(\Pi_t(k;b_t)-\Lambda_t^\star(k;b_t)\bigr).
\]
Applying the triangle inequality and Assumption~\ref{assump:score_approx} gives $|u_t(k)-U_t^\star(k;b_t)|\le \varepsilon_t^{\Delta}+\varepsilon_t^{\Psi}+\varepsilon_t^{\Pi}$. By Lemma~\ref{lem:budget_linearization}, $|\widetilde{Q}_t^\star(k)-U_t^\star(k;b_t)|=|\xi_t(k)|\le \beta_t\|g_t(k)\|+\frac{L_t}{2}\|g_t(k)\|^2$. The triangle inequality then gives
\[
|u_t(k)-\widetilde{Q}_t^\star(k)|\le \varepsilon_t^{\Delta}+\varepsilon_t^{\Psi}+\varepsilon_t^{\Pi}+\beta_t\|g_t(k)\|+\tfrac{L_t}{2}\|g_t(k)\|^2=\Gamma_t(k),
\]
which is Eq.~\eqref{eq:local_utility_bound}.
\end{proof}

Although Theorem~\ref{thm:controller_voi_local} is stated for the unnormalized utility, action selection uses the normalized score in Eq.~\eqref{eq:normalized_score}. The next corollary shows that the normalized ranking is stable when the oracle score margin dominates the approximation terms.

\begin{corollary}[Ranking consistency under a margin condition]
\label{cor:ranking_consistency}
If
\begin{equation}
r_t^\star(i)-r_t^\star(j)>\frac{\Gamma_t(i)+\Gamma_t(j)}{d_{\min}+\epsilon}
\label{eq:controller_regret_new}
\end{equation}
for some $d_{\min}>0$ lower-bounding all denominators, then $r_t(i)>r_t(j)$.
\end{corollary}

\begin{proof}
Define $r_t^\star(k):=[U_t^\star(k;b_t)]_+/(d_t(k;b_t)+\epsilon)$. Since $[\cdot]_+$ is $1$-Lipschitz,
\[
|r_t(k)-r_t^\star(k)|
=
\frac{|[u_t(k)]_+-[U_t^\star(k;b_t)]_+|}{d_t(k;b_t)+\epsilon}
\le
\frac{|u_t(k)-U_t^\star(k;b_t)|}{d_{\min}+\epsilon}
\le
\frac{\Gamma_t(k)}{d_{\min}+\epsilon}.
\]
If $r_t^\star(i)-r_t^\star(j)>(\Gamma_t(i)+\Gamma_t(j))/(d_{\min}+\epsilon)$, then
\[
r_t(i)\ge r_t^\star(i)-\frac{\Gamma_t(i)}{d_{\min}+\epsilon}>r_t^\star(j)+\frac{\Gamma_t(j)}{d_{\min}+\epsilon}\ge r_t(j),
\]
which proves the result.
\end{proof}

The ranking statement does not yet include the guard layer. Under the ranking-preserving guard condition, the same local bound implies a one-step value-gap guarantee for the implemented action. This result is conditional: it isolates the loss due to score approximation when the guard layer does not overturn the utility ordering relevant to the oracle action, and it is not a global optimality guarantee for arbitrary guard settings.

\begin{corollary}[One-step value gap under ranking-preserving guards]
\label{cor:value_gap}
Under the conditions of Theorem~\ref{thm:controller_voi_local} and the ranking-preserving guard condition in Assumption~\ref{assump:guard_compat}, the implemented action satisfies
\begin{equation}
Q_t^\star(k_t^{\mathrm{opt}})-Q_t^\star(k_t)\le 2\max_{k\in\mathcal{K}_t^{\mathrm{feas}}}\Gamma_t(k).
\label{eq:value_gap}
\end{equation}
\end{corollary}

\begin{proof}[Proof of Corollary~\ref{cor:value_gap}]
Let $\bar\Gamma_t:=\max_{k\in\mathcal{K}_t^{\mathrm{feas}}}\Gamma_t(k)$. From Eq.~\eqref{eq:local_utility_bound}, $\widetilde{Q}_t^\star(k)\le u_t(k)+\bar\Gamma_t$ and $\widetilde{Q}_t^\star(k)\ge u_t(k)-\bar\Gamma_t$ for every $k$. By Assumption~\ref{assump:guard_compat}, $u_t(k_t)\ge u_t(k_t^{\mathrm{opt}})$. Therefore,
\[
\widetilde{Q}_t^\star(k_t^{\mathrm{opt}})\le u_t(k_t^{\mathrm{opt}})+\bar\Gamma_t\le u_t(k_t)+\bar\Gamma_t\le \widetilde{Q}_t^\star(k_t)+2\bar\Gamma_t.
\]
Adding $V_t^{\mathrm{stop}}$ to both sides gives $Q_t^\star(k_t^{\mathrm{opt}})-Q_t^\star(k_t)\le 2\bar\Gamma_t$, which is Eq.~\eqref{eq:value_gap}.
\end{proof}

\begin{remark}
The ranking-consistency result covers the normalized raw score $r_t(k)$ before guards. Guards further modify the executable score $\widetilde{\mathcal{J}}_t(k)$ in action-specific ways, so oracle ranking of $r_t$ does not directly imply oracle ranking of $\widetilde{\mathcal{J}}_t$.
\end{remark}

\subsection{Answer-Time Finalization: Reward--Harm Decomposition and Optimality}
\label{app:answer_time_proofs}

We now collect the answer-time quantities used in Theorem~\ref{thm:finalizer_threshold}. Let $(x,\mathcal{T},a_{\mathrm{base}},a_{\mathrm{ref}},y)$ be jointly distributed according to the evaluation distribution, where $a_{\mathrm{ref}}=g_{\mathrm{ref}}(\mathcal{T},a_{\mathrm{base}})$, let $\mathcal{Z}$ denote the range of $z$, let $Z:=z(\mathcal{T},a_{\mathrm{base}},a_{\mathrm{ref}})$, and let an answer-replacement policy be a measurable map $\pi:\mathcal{Z}\to\{0,1\}$. We restrict attention to safe policies satisfying $\pi(z)=0$ for $z\notin\mathcal{S}_{\mathrm{safe}}$. Define $\Delta_{\mathrm{fin}}:=R(a_{\mathrm{ref}},y)-R(a_{\mathrm{base}},y)$ and
\begin{equation}
G^\star(z):=\mathbb{E}[(\Delta_{\mathrm{fin}})_+\mid Z=z], \qquad H^\star(z):=\mathbb{E}[(-\Delta_{\mathrm{fin}})_+\mid Z=z].
\label{eq:true_gain_harm}
\end{equation}

We now turn from search-time action selection to answer-time answer replacement. The first step is to express finalization as a gain--harm trade-off relative to the base trajectory answer.

\begin{lemma}[Reward--harm decomposition]
\label{lem:reward_harm}
For any safe answer-replacement policy $\pi$,
\begin{align}
\mathbb{E}[R(\hat a,y)]-\mathbb{E}[R(a_{\mathrm{base}},y)]
&=\mathbb{E}\!\left[\pi(Z)\bigl(G^\star(Z)-H^\star(Z)\bigr)\right],
\label{eq:reward_decomp}\\
\mathbb{E}[L_{\mathrm{harm}}(\hat a,a_{\mathrm{base}},y)]
&=\mathbb{E}\!\left[\pi(Z)H^\star(Z)\right].
\label{eq:harm_decomp}
\end{align}
\end{lemma}

\begin{proof}
Fix a safe policy $\pi$. By definition, $\hat a=a_{\mathrm{ref}}$ if $\pi(Z)=1$ and $\hat a=a_{\mathrm{base}}$ if $\pi(Z)=0$, so $R(\hat a,y)-R(a_{\mathrm{base}},y)=\pi(Z)\Delta_{\mathrm{fin}}$. Taking expectations, conditioning on $Z$, and using $\Delta_{\mathrm{fin}}=(\Delta_{\mathrm{fin}})_+-(-\Delta_{\mathrm{fin}})_+$ yields Eq.~\eqref{eq:reward_decomp}. Likewise, $L_{\mathrm{harm}}(\hat a,a_{\mathrm{base}},y)=\pi(Z)(-\Delta_{\mathrm{fin}})_+$, and conditioning on $Z$ yields Eq.~\eqref{eq:harm_decomp}.
\end{proof}

By Lemma~\ref{lem:reward_harm}, maximizing the answer-time part of Eq.~\eqref{eq:two_stage_problem} is equivalent to
\begin{equation}
\max_{\pi}\; \mathbb{E}\!\left[\pi(Z)\bigl(G^\star(Z)-H^\star(Z)\bigr)\right]
\quad\text{s.t.}\quad
\mathbb{E}\!\left[\pi(Z)H^\star(Z)\right]\le \rho_{\mathrm{harm}},
\qquad
\pi(z)=0\ \text{for } z\notin\mathcal{S}_{\mathrm{safe}},
\label{eq:fin_problem}
\end{equation}
where the maximum is over measurable $\{0,1\}$-valued policies. For $\eta>0$ and $\tau\in\mathbb{R}$, write $F_{\eta,\tau}^\star(z):=G^\star(z)-\eta H^\star(z)-\tau$.

\begin{lemma}[Pointwise maximizer of the penalized objective]
\label{lem:pointwise_finalizer}
Fix $\eta>0$ and $\tau\in\mathbb{R}$. Among all measurable policies satisfying $\pi(z)=0$ for $z\notin\mathcal{S}_{\mathrm{safe}}$, the maximizer of $\mathbb{E}[\pi(Z)F_{\eta,\tau}^\star(Z)]$ is
\begin{equation}
\pi_{\eta,\tau}^\star(z)=\mathbf{1}\!\left\{z\in\mathcal{S}_{\mathrm{safe}} \text{ and } F_{\eta,\tau}^\star(z)\ge 0\right\}.
\label{eq:pointwise_threshold_rule}
\end{equation}
\end{lemma}

\begin{proof}
Fix $z\in\mathcal{S}_{\mathrm{safe}}$. Because $\pi(z)\in\{0,1\}$, the pointwise contribution $\pi(z)F_{\eta,\tau}^\star(z)$ equals either $0$ or $F_{\eta,\tau}^\star(z)$. Hence the maximizing choice is $\pi(z)=1$ when $F_{\eta,\tau}^\star(z)\ge 0$ and $\pi(z)=0$ otherwise. Outside $\mathcal{S}_{\mathrm{safe}}$, feasibility forces $\pi(z)=0$. Since the objective is the expectation of this pointwise-separable integrand, the resulting rule is a global maximizer.
\end{proof}

\begin{assumption}[Strong duality and multiplier attainment]
\label{assump:dual}
There exists a multiplier $\gamma^\star\ge 0$ such that: (i) strong duality holds---the constrained problem in Eq.~\eqref{eq:fin_problem} and its Lagrangian relaxation have the same optimal value; (ii) the Lagrangian relaxation attains its maximum at $\gamma^\star$; (iii) the policy $\pi_{\gamma^\star}(z):=\mathbf{1}\{z\in\mathcal{S}_{\mathrm{safe}}\text{ and }G^\star(z)-(1+\gamma^\star)H^\star(z)\ge 0\}$ is primal-feasible, i.e., $\mathbb{E}[\pi_{\gamma^\star}(Z)H^\star(Z)]\le \rho_{\mathrm{harm}}$; and (iv) $\pi_{\gamma^\star}$ attains the constrained optimum, i.e., $\pi_{\gamma^\star}$ achieves the maximum of Eq.~\eqref{eq:fin_problem}.
\end{assumption}
This is a population-level regularity condition for the oracle constrained selection problem; the released deterministic finalizer below is a conservative feature-rule instantiation and is not claimed to know $G^\star$ or $H^\star$.

\begin{proof}[Proof of Theorem~\ref{thm:finalizer_threshold}]
By Lemma~\ref{lem:reward_harm}, the safe answer-replacement problem is Eq.~\eqref{eq:fin_problem}. For any multiplier $\gamma\ge 0$, its Lagrangian is
\begin{equation}
\mathcal{L}(\pi,\gamma):=\mathbb{E}\!\left[\pi(Z)\bigl(G^\star(Z)-H^\star(Z)\bigr)\right]-\gamma\left(\mathbb{E}\!\left[\pi(Z)H^\star(Z)\right]-\rho_{\mathrm{harm}}\right)
\end{equation}
\begin{equation}
=\mathbb{E}\!\left[\pi(Z)\bigl(G^\star(Z)-(1+\gamma)H^\star(Z)\bigr)\right]+\gamma\rho_{\mathrm{harm}}.
\label{eq:lagrangian_finalizer}
\end{equation}
For fixed $\gamma$, the term $\gamma\rho_{\mathrm{harm}}$ is constant in $\pi$, so maximizing $\mathcal{L}(\pi,\gamma)$ is equivalent to maximizing $\mathbb{E}[\pi(Z)(G^\star(Z)-(1+\gamma)H^\star(Z))]$ subject to the safe-set constraint. Applying Lemma~\ref{lem:pointwise_finalizer} with $\eta=1+\gamma$ and $\tau=0$ shows that the pointwise maximizer is
\begin{equation}
\pi_\gamma(z)=\mathbf{1}\!\left\{z\in\mathcal{S}_{\mathrm{safe}} \text{ and } G^\star(z)-(1+\gamma)H^\star(z)\ge 0\right\}.
\label{eq:gamma_policy}
\end{equation}
By Assumption~\ref{assump:dual}(i)--(ii), there exists $\gamma^\star\ge 0$ at which the Lagrangian relaxation attains the constrained optimum. By Assumption~\ref{assump:dual}(iii), $\pi_{\gamma^\star}$ is primal-feasible. By Assumption~\ref{assump:dual}(iv), $\pi_{\gamma^\star}$ achieves the primal optimum. Setting $\eta^\star:=1+\gamma^\star$ gives $\eta^\star\ge 1$, and substituting into Eq.~\eqref{eq:gamma_policy} yields exactly the threshold rule in Eq.~\eqref{eq:optimal_safe_writeback}.
\end{proof}

\subsection{Released Rule-Based Finalizer: Exact Rule Characterization}
\label{app:det_finalizer_rule}

\begin{remark}
The released deterministic finalizer is \emph{not} a runtime instantiation of the oracle threshold rule in Theorem~\ref{thm:finalizer_threshold}: the oracle quantities $G^\star(z)$ and $H^\star(z)$ are unknown at inference time. Instead, it is a conservative deterministic rule over explicit feature conditions, constructed as a special case of the threshold form. Theorem~\ref{thm:finalizer_threshold} characterizes the performance upper envelope when $G^\star,H^\star$ are known; Proposition~\ref{prop:det_finalizer} characterizes the rule that the released code actually executes below that envelope. The empirical gap between the two is an open question.
\end{remark}

The oracle threshold rule uses unknown conditional quantities $G^\star$ and $H^\star$. The reported system therefore uses a deterministic conservative feature rule. The next proposition characterizes that rule at the paper level.

Define the following paper-level features. Let $m_{\mathrm{ref}}$ indicate that a refined candidate is available and passes the preliminary plausibility filter; let $c_{\mathrm{risk}}$ be the refinement-risk category; let $n_{\mathrm{dec}}$ be the number of decomposition steps in the trajectory; let $q_{\mathrm{type}}$ be the detected question/answer type; let $q_{\mathrm{slot}}$ be the detected slot type; let $\Delta_{\mathrm{sup}}$ be the support gain of the refined candidate over the base answer; and let $\ell_{\mathrm{base}}$ and $\ell_{\mathrm{ref}}$ be the token lengths of the base and refined candidates. Let $\mathcal{C}_{\mathrm{block}}$ denote the high-risk refinement categories blocked by the finalizer, let $\mathcal{Q}_{\mathrm{bin}}$ denote yes/no and binary-choice cases, and let $\mathcal{Q}_{\mathrm{typed}}$ denote typed-slot cases such as capacity, date, and year range. Define four branch condition sets:
\begin{align*}
\mathcal{B}_{\mathrm{bin}} &:= \{q_{\mathrm{type}}\in\mathcal{Q}_{\mathrm{bin}}\},\\
\mathcal{B}_{\mathrm{typed}} &:= \{q_{\mathrm{slot}}\in\mathcal{Q}_{\mathrm{typed}},\;\Delta_{\mathrm{sup}}\ge 0.50,\;\ell_{\mathrm{ref}}\le \ell_{\mathrm{base}}+1\},\\
\mathcal{B}_{\mathrm{explicit}} &:= \{\text{explicit factoid-slot case},\;\Delta_{\mathrm{sup}}\ge 0,\;\ell_{\mathrm{ref}}\le \ell_{\mathrm{base}}+3\},\\
\mathcal{B}_{\mathrm{compact}} &:= \{\Delta_{\mathrm{sup}}>0,\;\ell_{\mathrm{ref}}\le \ell_{\mathrm{base}}+2\}.
\end{align*}
Define the priority-ordered branch selector:
\[
\mathcal{B}_{\mathrm{accept}} :=
\begin{cases}
\mathcal{B}_{\mathrm{bin}}, & \text{if } q_{\mathrm{type}}\in\mathcal{Q}_{\mathrm{bin}},\\
\mathcal{B}_{\mathrm{typed}}, & \text{if } q_{\mathrm{type}}\notin\mathcal{Q}_{\mathrm{bin}} \text{ and } q_{\mathrm{slot}}\in\mathcal{Q}_{\mathrm{typed}},\\
\mathcal{B}_{\mathrm{explicit}}, & \text{if } q_{\mathrm{type}}\notin\mathcal{Q}_{\mathrm{bin}},\;q_{\mathrm{slot}}\notin\mathcal{Q}_{\mathrm{typed}},\;\text{and the case is an explicit factoid slot},\\
\mathcal{B}_{\mathrm{compact}}, & \text{otherwise.}
\end{cases}
\]
The accept set is not an unconditional union of the branches: when an earlier applicable branch fails its safety test, the finalizer abstains instead of falling through to a lower-priority branch. The accept set is
\[
\mathcal{S}_{\mathrm{det}}
:=
\{m_{\mathrm{ref}}=1,\; c_{\mathrm{risk}}\notin\mathcal{C}_{\mathrm{block}},\; n_{\mathrm{dec}}=0\}
\cap \mathcal{B}_{\mathrm{accept}} .
\]

\begin{proposition}[Exact rule of the released deterministic finalizer]
\label{prop:det_finalizer}
Let $z$ denote the finalization feature vector extracted from the completed trajectory and the two candidate answers. The deterministic finalizer used in our experiments is equivalent to the indicator rule $\sigma_{\mathrm{det}}(z)=\mathbf{1}\{z\in\mathcal{S}_{\mathrm{det}}\}$ for the priority-ordered safe set $\mathcal{S}_{\mathrm{det}}$. Therefore, the final answer is $a_{\mathrm{ref}}$ when $z\in\mathcal{S}_{\mathrm{det}}$ and $a_{\mathrm{base}}$ otherwise.
\end{proposition}

\begin{proof}
Trace the rule in its evaluation order. The first gate requires that a refined candidate is available and passes the preliminary plausibility filter, which enforces $m_{\mathrm{ref}}=1$. The second gate blocks high-risk refinement categories, enforcing $c_{\mathrm{risk}}\notin\mathcal{C}_{\mathrm{block}}$. The third gate abstains on decomposition-heavy trajectories, enforcing $n_{\mathrm{dec}}=0$. Conditional on these gates, the branch selector is evaluated in priority order. Binary and yes/no cases are handled first, giving the branch set $\mathcal{B}_{\mathrm{bin}}$. If the case is not binary but has a typed slot, the typed-slot safety test is applied; acceptance requires $\Delta_{\mathrm{sup}}\ge 0.50$ and $\ell_{\mathrm{ref}}\le \ell_{\mathrm{base}}+1$, and failure causes abstention rather than falling through. If no earlier branch applies and the case is an explicit factoid slot, acceptance requires nonnegative support gain and $\ell_{\mathrm{ref}}\le \ell_{\mathrm{base}}+3$; again, failure causes abstention. The final fallback accepts only compact support-improving refinements, requiring $\Delta_{\mathrm{sup}}>0$ and $\ell_{\mathrm{ref}}\le \ell_{\mathrm{base}}+2$. Thus the rule accepts exactly when the three gates hold and the priority-ordered branch selector accepts, which is precisely $z\in\mathcal{S}_{\mathrm{det}}$.
\end{proof}

\subsection{Plug-in Approximation of the Finalizer}
\label{app:plugin_finalizer}

The previous proposition describes the released rule. For completeness, we also record a generic perturbation result showing how approximate gain--harm scores affect threshold decisions: Lemma~\ref{lem:fin_score_perturb} controls the pointwise score error, and Theorem~\ref{thm:plugin_excess} converts it into a penalized-value bound.

Fix $\eta>0$ and $\tau\in\mathbb{R}$. Write $F_{\eta,\tau}^\star(z):=G^\star(z)-\eta H^\star(z)-\tau$ and $F_{\eta,\tau}(z):=G(z)-\eta H(z)-\tau$, define the oracle rule by $\pi_{\eta,\tau}^\star(z):=\mathbf{1}\{z\in\mathcal{S}_{\mathrm{safe}} \text{ and } F_{\eta,\tau}^\star(z)\ge 0\}$, and define the plug-in rule by $\hat\pi_{\eta,\tau}(z):=\mathbf{1}\{z\in\mathcal{S}_{\mathrm{safe}} \text{ and } F_{\eta,\tau}(z)\ge 0\}$.

\begin{lemma}[Finalizer score perturbation]
\label{lem:fin_score_perturb}
If $|G(z)-G^\star(z)|\le \delta_G$ and $|H(z)-H^\star(z)|\le \delta_H$ for all $z\in\mathcal{S}_{\mathrm{safe}}$, then
\[
|F_{\eta,\tau}(z)-F_{\eta,\tau}^\star(z)|\le \delta_G+\eta\,\delta_H
\]
for all $z\in\mathcal{S}_{\mathrm{safe}}$.
\end{lemma}

\begin{proof}
For $z\in\mathcal{S}_{\mathrm{safe}}$, write $F_{\eta,\tau}(z)-F_{\eta,\tau}^\star(z)=\bigl(G(z)-G^\star(z)\bigr)-\eta\bigl(H(z)-H^\star(z)\bigr)$ and apply the triangle inequality.
\end{proof}

\begin{theorem}[Plug-in excess penalized value]
\label{thm:plugin_excess}
Fix $\eta>0$ and $\tau\in\mathbb{R}$, and let $\delta:=\delta_G+\eta\,\delta_H$. Under the bounds of Lemma~\ref{lem:fin_score_perturb},
\[
\mathbb{E}\!\left[\pi_{\eta,\tau}^\star(Z)F_{\eta,\tau}^\star(Z)\right]-\mathbb{E}\!\left[\hat\pi_{\eta,\tau}(Z)F_{\eta,\tau}^\star(Z)\right]
\le
\mathbb{E}\!\left[|F_{\eta,\tau}^\star(Z)|\mathbf{1}\!\left\{|F_{\eta,\tau}^\star(Z)|\le \delta,\; Z\in\mathcal{S}_{\mathrm{safe}}\right\}\right].
\]
If, in addition, there exist constants $C>0$ and $\alpha>0$ such that
\[
\mathbb{P}\!\left(|F_{\eta,\tau}^\star(Z)|\le u,\; Z\in\mathcal{S}_{\mathrm{safe}}\right)\le C u^\alpha \qquad \text{for all } u\ge 0,
\]
then
\[
\mathbb{E}\!\left[\pi_{\eta,\tau}^\star(Z)F_{\eta,\tau}^\star(Z)\right]-\mathbb{E}\!\left[\hat\pi_{\eta,\tau}(Z)F_{\eta,\tau}^\star(Z)\right]\le C\delta^{1+\alpha}.
\]
\end{theorem}

\begin{proof}
Let $D:=\{Z\in\mathcal{S}_{\mathrm{safe}}:\pi_{\eta,\tau}^\star(Z)\neq \hat\pi_{\eta,\tau}(Z)\}$. Because $\pi_{\eta,\tau}^\star$ is the pointwise maximizer of $\pi(Z)F_{\eta,\tau}^\star(Z)$ over $\pi(Z)\in\{0,1\}$, the difference
\[
\mathbb{E}\!\left[\pi_{\eta,\tau}^\star(Z)F_{\eta,\tau}^\star(Z)\right]-\mathbb{E}\!\left[\hat\pi_{\eta,\tau}(Z)F_{\eta,\tau}^\star(Z)\right]
\]
is supported on $D$ and equals $\mathbb{E}[|F_{\eta,\tau}^\star(Z)|\mathbf{1}_D]$. On $D$, the oracle and plug-in scores have opposite signs, so $|F_{\eta,\tau}^\star(Z)|\le |F_{\eta,\tau}^\star(Z)-F_{\eta,\tau}(Z)|$. Applying Lemma~\ref{lem:fin_score_perturb} gives $D\subseteq\{|F_{\eta,\tau}^\star(Z)|\le \delta,\; Z\in\mathcal{S}_{\mathrm{safe}}\}$ and therefore the first bound. For the second bound, use $|F_{\eta,\tau}^\star(Z)|\le \delta$ on the boundary set to obtain
\[
\mathbb{E}\!\left[|F_{\eta,\tau}^\star(Z)|\mathbf{1}\!\left\{|F_{\eta,\tau}^\star(Z)|\le \delta,\; Z\in\mathcal{S}_{\mathrm{safe}}\right\}\right]
\le
\delta\,\mathbb{P}\!\left(|F_{\eta,\tau}^\star(Z)|\le \delta,\; Z\in\mathcal{S}_{\mathrm{safe}}\right)\le C\delta^{1+\alpha}.
\]
\end{proof}

\subsection{Finite Termination of Algorithm~\ref{alg:two_stage_inference}}
\label{app:termination}

Finally, Corollary~\ref{cor:termination} records that the inference procedure is finite under a minimal budget-decrement condition. The sum $B_{\mathrm{tool}}+B_{\mathrm{tok}}$ is used only as a bookkeeping potential over two nonnegative budget coordinates, not as a single physical cost metric.

\begin{corollary}[Finite termination]
\label{cor:termination}
Algorithm~\ref{alg:two_stage_inference} terminates in at most $\lceil(B_{\mathrm{tool}}+B_{\mathrm{tok}})/\zeta\rceil$ iterations under the minimum-decrement condition below.
\end{corollary}

\begin{lemma}[Finite termination lemma]
\label{lem:termination}
Assume that every executed action has nonnegative cost in each budget coordinate and that, whenever the \texttt{while}-loop in Algorithm~\ref{alg:two_stage_inference} continues, at least one coordinate of the remaining budget decreases by at least a fixed amount $\zeta>0$. If the initial budget $b_0=(B_{\mathrm{tool}},B_{\mathrm{tok}})$ is finite, then the \texttt{while}-loop terminates after at most $\left\lceil \frac{B_{\mathrm{tool}}+B_{\mathrm{tok}}}{\zeta}\right\rceil$ iterations. This proves Corollary~\ref{cor:termination}.
\end{lemma}

\begin{proof}
Write $S_t:=b_{\mathrm{tool},t}+b_{\mathrm{tok},t}$. Since each continuing iteration decreases at least one coordinate by at least $\zeta$ and neither coordinate increases, we have $S_{t+1}\le S_t-\zeta$. Because $S_0=B_{\mathrm{tool}}+B_{\mathrm{tok}}$ and $S_t\ge 0$ for all $t$, the loop can continue at most $\left\lceil(B_{\mathrm{tool}}+B_{\mathrm{tok}})/\zeta\right\rceil$ times. The post-loop extraction and finalization steps are finite, so Algorithm~\ref{alg:two_stage_inference} always returns in finitely many steps.
\end{proof}

\nolinenumbers

\section{Experimental Protocol and Implementation Details}
\label{app:experimental_details}

\paragraph{Fixed controller implementation.}
All audited experiments were run on a workstation with 4 NVIDIA RTX 5880 Ada GPUs. The controller itself is deterministic and lightweight; the dominant cost comes from LLM inference and retrieval calls. The Stage-1 controller is a fixed, training-free scoring rule added on top of the BAVT search procedure. BAVT provides the planner, generator, critic interface, search procedure, and remaining-budget state; our added component is the task-level VOI score action scorer over \search{}, \decompose{}, and \answer{}. No coefficients are learned from labeled trajectories, no regression model is fitted, and no benchmark-specific hyperparameter search is performed over the reported audit cells. The same controller parameters are used across all benchmarks, budget levels, and LLM backbones.

\paragraph{Budget pressure and action scoring.}
We use the normalized remaining-budget pressure
\[
\rho_t
=
1-\min\!\left\{
\frac{b_{\mathrm{tool},t}}{B_{\mathrm{tool}}},
\frac{b_{\mathrm{tok},t}}{B_{\mathrm{tok}}}
\right\},
\]
clipped to $[0,1]$. This scalar increases as either the tool-call or output-token budget becomes tight. It enters both the budget-dependent penalty term $\Pi_t(k;b_t)$ and the action-specific cost scale $d_t(k;b_t)$: additional \search{} and \decompose{} steps are penalized more strongly under tight budgets, while \answer{} becomes relatively more attractive when the trajectory has sufficient support.

\paragraph{Fixed coefficients and guards.}
The main fixed Stage-1 coefficients are a cost-penalty scale of $0.7$, a decomposition bonus of $0.14$, and an early-answer penalty of $0.18$. These values remain fixed in all reported experiments. After forming the utility numerator, the controller divides positive utility by an action-specific cost scale and then applies deterministic guards. The guards suppress premature answer commitment under weak support, suppress decomposition for low-compositionality or factoid-like questions, downweight repeated decomposition after stagnation, and enforce minimum search for sufficiently compositional states. These guards only change the executable action score; they do not change the retrieval backend, generator, sample order, or budget accounting protocol.

\paragraph{Stage-2 finalization.}
The Stage-2 finalizer is a deterministic answer-selection rule over the completed trajectory. It compares the base answer $a_{\mathrm{base}}$ with a refined candidate $a_{\mathrm{ref}}$ derived from the same trace and accepts the refined candidate only under low-risk answer-form conditions, such as yes/no polarity repair, binary-choice repair, typed-slot correction, or supported factoid completion. It abstains under unresolved bridge structure, unresolved comparative reasoning, or missing direct support. This stage adds no tool calls and issues no additional LLM call during finalization.

\section{Full Qwen3-32B Main Results}
\label{app:main_results_qwen32b}

\begin{table*}[hbpt]
\centering
\caption{\textbf{Main strict-budget results on Qwen3-32B.}
Each method reports \texttt{EM/F1} under the same symmetric strict dual-budget protocol with tool-call and output-token caps
$(1,100)$, $(2,200)$, $(2,300)$, and $(3,500)$.
\method{} denotes our full two-stage budget-control method.
The 1st, 2nd, and 3rd best results are highlighted in \colorbox{colorfirst}{\textbf{first}}, \colorbox{colorsecond}{second}, and \colorbox{colorthird}{third} colors. Cell background colors are ranked by the average of EM and F1 scores. Best EM and F1 are independently bolded. $\uparrow$ indicates higher is better.}
\label{tab:main_results}
\small
\setlength{\tabcolsep}{4.6pt}
{\fontfamily{ptm}\selectfont
\resizebox{\textwidth}{!}{%
\begin{tabular}{llccccc}
\toprule
\multirow{2}{*}{\textbf{Benchmark}} & \multirow{2}{*}{\textbf{Budget}}
& \textbf{BATS} & \textbf{BAVT} & \textbf{\aflow} & \textbf{\searcho} & \textbf{\method} \\
\cmidrule(lr){3-7}
& & \textbf{EM/F1} $\uparrow$ & \textbf{EM/F1} $\uparrow$ & \textbf{EM/F1} $\uparrow$ & \textbf{EM/F1} $\uparrow$ & \textbf{EM/F1} $\uparrow$ \\
\midrule
\bamboogle & \budgetlow
& 0.02/0.03 & \cellcolor{colorsecond} 0.11/0.17 & \cellcolor{colorthird} 0.03/0.03 & 0.00/0.02 & \cellcolor{colorfirst} \textbf{0.15}/\textbf{0.21} \\
\bamboogle & \budgetlowermid
& \cellcolor{colorthird} 0.05/0.06 & \cellcolor{colorsecond} 0.21/0.28 & \cellcolor{colorfirst} \textbf{0.33}/\textbf{0.42} & 0.03/0.05 & \cellcolor{colorfirst} \textbf{0.33}/0.42 \\
\bamboogle & \budgetuppermid
& \cellcolor{colorfirst} \textbf{0.42}/\textbf{0.54} & \cellcolor{colorsecond} 0.33/0.42 & \cellcolor{colorsecond} 0.33/0.42 & 0.20/0.24 & \cellcolor{colorfirst} \textbf{0.43}/0.53 \\
\bamboogle & \budgethigh
& \cellcolor{colorsecond} 0.45/0.58 & 0.39/0.47 & \cellcolor{colorthird} 0.37/0.50 & 0.33/0.44 & \cellcolor{colorfirst} \textbf{0.48}/\textbf{0.62} \\
\midrule
\hotpot & \budgetlow
& \cellcolor{colorthird} 0.01/0.02 & \cellcolor{colorsecond} 0.11/0.14 & 0.01/0.01 & 0.00/0.00 & \cellcolor{colorfirst} \textbf{0.14}/\textbf{0.17} \\
\hotpot & \budgetlowermid
& 0.02/0.04 & \cellcolor{colorsecond} 0.28/0.34 & 0.01/0.01 & \cellcolor{colorthird} 0.03/0.04 & \cellcolor{colorfirst} \textbf{0.36}/\textbf{0.41} \\
\hotpot & \budgetuppermid
& \cellcolor{colorfirst} \textbf{0.40}/\textbf{0.49} & \cellcolor{colorthird} 0.34/0.41 & 0.01/0.01 & 0.09/0.14 & \cellcolor{colorsecond} 0.39/0.47 \\
\hotpot & \budgethigh
& \cellcolor{colorfirst} \textbf{0.43}/\textbf{0.52} & \cellcolor{colorsecond} 0.34/0.40 & \cellcolor{colorthird} 0.29/0.36 & 0.22/0.31 & \cellcolor{colorfirst} \textbf{0.43}/0.52 \\
\midrule
\musique & \budgetlow
& \cellcolor{colorsecond} 0.02/0.03 & \cellcolor{colorfirst} 0.11/\textbf{0.17} & \cellcolor{colorthird} 0.01/0.01 & \cellcolor{colorthird} 0.00/0.02 & \cellcolor{colorfirst} \textbf{0.12}/0.16 \\
\musique & \budgetlowermid
& 0.06/0.08 & \cellcolor{colorsecond} 0.26/0.33 & 0.03/0.04 & \cellcolor{colorthird} 0.03/0.06 & \cellcolor{colorfirst} \textbf{0.28}/\textbf{0.35} \\
\musique & \budgetuppermid
& \cellcolor{colorfirst} \textbf{0.35}/\textbf{0.43} & \cellcolor{colorthird} 0.26/0.34 & 0.06/0.08 & 0.16/0.19 & \cellcolor{colorsecond} 0.33/0.40 \\
\musique & \budgethigh
& \cellcolor{colorsecond} 0.37/0.48 & \cellcolor{colorthird} 0.34/0.42 & 0.11/0.16 & 0.29/0.35 & \cellcolor{colorfirst} \textbf{0.40}/\textbf{0.50} \\
\midrule
\twowiki & \budgetlow
& 0.02/0.02 & \cellcolor{colorsecond} 0.08/0.10 & \cellcolor{colorthird} 0.07/0.07 & 0.00/0.04 & \cellcolor{colorfirst} \textbf{0.09}/\textbf{0.11} \\
\twowiki & \budgetlowermid
& 0.05/0.05 & \cellcolor{colorthird} 0.28/0.33 & \cellcolor{colorfirst} \textbf{0.47}/\textbf{0.55} & 0.03/0.04 & \cellcolor{colorsecond} 0.44/0.49 \\
\twowiki & \budgetuppermid
& \cellcolor{colorfirst} 0.53/\textbf{0.64} & 0.40/0.45 & \cellcolor{colorthird} 0.47/0.55 & 0.15/0.20 & \cellcolor{colorsecond} \textbf{0.54}/0.62 \\
\twowiki & \budgethigh
& \cellcolor{colorfirst} 0.62/\textbf{0.73} & \cellcolor{colorthird} 0.56/0.63 & \cellcolor{colorsecond} 0.58/0.70 & 0.27/0.36 & \cellcolor{colorfirst} \textbf{0.64}/0.71 \\
\midrule
\multicolumn{2}{l}{\textbf{Average}}
& \cellcolor{colorthird} 0.24/0.30 & \cellcolor{colorsecond} 0.28/0.34 & 0.20/0.24 & 0.11/0.16 & \cellcolor{colorfirst} \textbf{0.35}/\textbf{0.42} \\
\bottomrule
\end{tabular}}
}
\vspace{-0.6em}
\end{table*}

\section{Qwen3.5-122B Backbone-Sensitivity Results}
\label{app:qwen35_scaling}

Figure~\ref{fig:scaling_qwen35} reports the full budget scaling curves
for the Qwen3.5-122B backbone across all four benchmarks.
Compared with Qwen3-32B and GPT-5.4-Mini (Figure~\ref{fig:scaling_cross_model}),
Qwen3.5-122B exhibits a more mixed competitive regime:
\bats{} and \vanilla{} lead in several cells, particularly at upper budgets,
while \method{} maintains an advantage at lower budgets where
explicit budget penalty provides the largest marginal benefit.

\begin{figure}[hbpt]
\centering
\includegraphics[width=\textwidth]{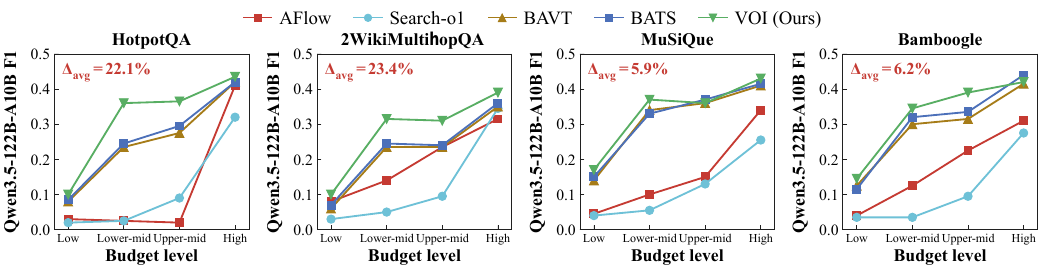}
\caption{\textbf{Budget scaling curves for Qwen3.5-122B across four datasets.}
Each column corresponds to one benchmark; all methods are evaluated under
the shared dual-budget protocol. See Figure~\ref{fig:scaling_cross_model}
for Qwen3-32B and GPT-5.4-Mini results.}
\label{fig:scaling_qwen35}
\end{figure}

\section{Stage-1 Component Ablation Protocol}
\label{app:stage1_component_ablation}

The ablation in Section~\ref{subsec:stage1_component_ablation} is designed to match the Stage~1 controller in Section~\ref{subsec:controller}. Each variant removes exactly one component from the three-stage scoring pipeline while keeping the search procedure, prompts, retrieval backend, sample set, and hard-budget audit unchanged. The \textit{w/o penalty} variant removes the budget-dependent penalty term $\Pi_t(k;b_t)$ from Eq.~\eqref{eq:controller_score}; \textit{w/o normalization} replaces Eq.~\eqref{eq:normalized_score} with the unnormalized score $[u_t(k)]_+$; \textit{w/o structural} removes $\Psi_t(k)$ from Eq.~\eqref{eq:controller_score}; and \textit{w/o guards} bypasses $\mathfrak{G}_t$ in Eq.~\eqref{eq:guarded_score}. The full benchmark-level results are reported in Table~\ref{tab:stage1_component_ablation_qwen32b_full} in the main text.

\section{Additional Audit Results and Usage Diagnostics}
\label{app:scope}

\paragraph{Token accounting.}
For BAVT-family runs, \texttt{avg API tokens} includes prompt and completion tokens across planner, generator, and critic calls. \texttt{avg budget output tokens} is the quantity debited from the controller-side budget and reflects only the output-token budgeted component. All methods in the main table are evaluated under the same dual-budget constraint; the budget output token column reflects the output tokens charged against the budget in each case.

\paragraph{Method-specific audit notes.}
\aflow{}, \searcho{}, \bats{}, \vanilla{}, and \method{} are all scored under the same hard dual-budget audit used in the main table: the same Search-R1 retrieval backend, question-only queries, \texttt{top\_k=5}, and the same tool-call and output-token caps. Any example whose realized tool calls or output tokens exceed the target budget is counted as failed for that cell. For \aflow{}, this audit is implemented over replayed workflows: a replayed sample contributes to the scored cell only if its realized tool calls and answer output tokens remain within the target cap. \searcho{} and \bats{} are reported under the same audit semantics, and \vanilla{}/\method{} share the same BAVT-family search backbone under the same budget ladder.

\paragraph{Main-table comparator set.}
The primary empirical comparator set in this paper is \aflow{}, \bats{}, \searcho{}, \vanilla{}, and \method{}. All five appear in the main table because all five are evaluated under the same hard tool/output-token audit. 

\begin{table*}[t]
\caption{\textbf{Cross-backbone macro deltas for VOI against four baselines.}
$\Delta$F1 and $\Delta$EM are averaged over the 48 audited cells
(3 backbones $\times$ 4 benchmarks $\times$ 4 budgets), and their 95\% CI is a
cell-level bootstrap confidence interval.}
\label{tab:crossmodel_macro_delta_em_f1}
\centering
\small
\setlength{\tabcolsep}{0pt}
\begin{tabular*}{\textwidth}{@{\extracolsep{\fill}}lcccc@{}}
\toprule
Comparison & $\Delta$EM (est.) & 95\% CI & $\Delta$F1 & 95\% CI \\
\midrule
VOI vs AFlow     & 0.2456 & [0.2100, 0.2819] & 0.3021 & [0.2556, 0.3491] \\
VOI vs BATS      & 0.1081 & [0.0437, 0.1725] & 0.1219 & [0.0439, 0.1999] \\
VOI vs Search-o1 & 0.2734 & [0.2360, 0.3112] & 0.2993 & [0.2625, 0.3370] \\
VOI vs BAVT      & 0.0492 & [0.0382, 0.0603] & 0.0530 & [0.0423, 0.0639] \\
\bottomrule
\end{tabular*}
\end{table*}

\paragraph{Feasible-only usage audit.}
Table~\ref{tab:qwen32_feasible_usage_cellwise} reports realized usage on the feasible subset of each dataset-budget cell. For scoring, all main-table methods use the same hard tool/output-token audit. This table is a secondary feasibility-conditioned diagnostic: for each benchmark and budget level, we retain only examples satisfying both caps and then average realized tool calls and tokens over the retained subset. Different methods can therefore have different feasible subsets. For \method{}, \bats{}, \vanilla{}, and \searcho{}, tokens are controller-debited output tokens. In these Qwen3-32B runs, \aflow{} feasibility is mainly tool-cap limited.

% NOTE: Synthetic zero-imputed usage table for pattern analysis only.
% Do not report these imputed cells as audited experimental results.

\begin{table*}[t]
\centering
\scriptsize
\caption{\textbf{Feasible-only average resource usage by dataset and budget.}
Each method cell reports average tool calls / average output tokens, i.e., \texttt{Avg. Tools / Avg. Tok}, computed only over examples that remain feasible under the corresponding tool/output-token cap.
Budget labels follow the four-level budget ladder defined in Section~\ref{subsec:setup}.
For scoring, all methods are subject to the same hard tool/output-token caps. The table is a feasible-only usage diagnostic rather than a scoring rule: token entries report the output tokens available from each audited execution. For \aflow{}, which is evaluated through replayed workflows, the reported token entry corresponds to answer-output tokens from the replay.}
\label{tab:qwen32_feasible_usage_cellwise}
\setlength{\tabcolsep}{3.8pt}
\resizebox{\textwidth}{!}{%
\begin{tabular}{llccccc}
\toprule
\multirow{2}{*}{Benchmark} & \multirow{2}{*}{Budget}
& \method{} (Ours) & BATS & \vanilla{} & \aflow{}$^\dagger$ & \searcho{} \\
\cmidrule(lr){3-7}
& & \multicolumn{5}{c}{\texttt{Avg. Tools / Avg. Tokens}} \\
\midrule
\bamboogle & \budgetlow      & 0.97 / 73.7 & 0.93 / 86.4 & 0.94 / 69.7 & 1.00 / 2.0 & 0.02 / 100.0 \\
\bamboogle & \budgetlowermid & 1.61 / 138.8 & 1.00 / 190.9 & 1.43 / 135.4 & 1.96 / 7.8 & 0.69 / 193.9 \\
\bamboogle & \budgetuppermid & 1.64 / 171.3 & 1.58 / 240.2 & 1.62 / 179.7 & 1.96 / 7.8 & 0.87 / 249.6 \\
\bamboogle & \budgethigh     & 1.91 / 195.9 & 1.91 / 268.5 & 1.89 / 211.3 & 2.19 / 7.0 & 1.06 / 350.5 \\
\midrule
\hotpot & \budgetlow      & 0.92 / 76.6 & 0.88 / 84.9 & 0.84 / 67.7 & 1.00 / 2.0 & 0.04 / 100.0 \\
\hotpot & \budgetlowermid & 1.49 / 134.2 & 1.00 / 189.6 & 1.56 / 138.6 & 1.00 / 2.0 & 0.63 / 193.7 \\
\hotpot & \budgetuppermid & 1.61 / 162.7 & 1.62 / 241.2 & 1.61 / 158.8 & 1.00 / 2.0 & 0.93 / 262.1 \\
\hotpot & \budgethigh     & 1.91 / 197.7 & 1.91 / 269.0 & 1.88 / 218.9 & 2.98 / 8.9 & 1.09 / 361.5 \\
\midrule
\musique & \budgetlow      & 0.95 / 64.6 & 0.91 / 82.7 & 0.92 / 67.3 & 1.00 / 2.1 & 0.03 / 100.0 \\
\musique & \budgetlowermid & 1.60 / 141.1 & 1.00 / 193.8 & 1.46 / 137.3 & 1.48 / 4.6 & 0.68 / 192.5 \\
\musique & \budgetuppermid & 1.62 / 174.5 & 1.49 / 239.8 & 1.55 / 178.9 & 1.76 / 5.8 & 0.88 / 251.2 \\
\musique & \budgethigh     & 1.85 / 191.5 & 1.95 / 278.7 & 1.87 / 226.0 & 2.21 / 7.2 & 1.09 / 364.4 \\
\midrule
\twowiki & \budgetlow      & 0.81 / 74.8 & 0.86 / 85.6 & 0.73 / 69.7 & 1.00 / 2.3 & 0.13 / 100.0 \\
\twowiki & \budgetlowermid & 1.73 / 140.7 & 1.17 / 189.8 & 1.64 / 132.9 & 1.90 / 6.8 & 0.50 / 193.6 \\
\twowiki & \budgetuppermid & 1.82 / 166.2 & 1.91 / 249.7 & 1.72 / 176.9 & 1.90 / 6.6 & 1.00 / 268.5 \\
\twowiki & \budgethigh     & 2.04 / 193.9 & 2.09 / 271.4 & 1.91 / 222.2 & 2.14 / 6.1 & 1.22 / 376.4 \\
\bottomrule
\end{tabular}}
\vspace{-0.6em}
\end{table*}

\end{document}